\documentclass[letterpaper]{article} 
\usepackage[preprint]{aaai2027}  
\usepackage[hyphens]{url}  
\usepackage{graphicx} 
\usepackage{natbib}  
\usepackage{caption} 
\usepackage{algorithm}
\usepackage{algorithmic}
\usepackage{amsthm}
\usepackage{amsmath}
\usepackage{amssymb}
\usepackage{enumitem}
\usepackage{atnguyen_new}
\usepackage{soul}
\usepackage{mathtools}
\usepackage{tikz}

\usepackage{newfloat}
\usepackage{listings}
\DeclareCaptionStyle{ruled}{labelfont=normalfont,labelsep=colon,strut=off} 
\floatstyle{ruled}
\newfloat{listing}{tb}{lst}{}
\floatname{listing}{Listing}

\usepackage{booktabs}

\title{Tight Bounds for Data-driven Multiple Hyper-parameter Tuning \\ with Structured Loss Function}
\author {
    Anh Tuan Nguyen\textsuperscript{\rm 1},
    Viet Anh Nguyen\textsuperscript{\rm 2}
}
\affiliations {
    \textsuperscript{\rm 1}Carnegie Mellon University\\
    \textsuperscript{\rm 2}The Chinese University of Hong Kong\\
    atnguyen@cs.cmu.edu, nguyen@se.cuhk.edu.hk
}

\begin{document}

\maketitle

\begin{abstract}
    Data-driven algorithm design frames hyperparameter tuning as a statistical learning problem, but establishing generalization guarantees remains challenging due to the implicit, non-smooth dependence of model performance on hyperparameters. Existing multi-dimensional bounds under piecewise-polynomial assumptions remain theoretically loose and lack comprehensive lower bounds. We resolve this by establishing tight pseudo-dimension bounds for multi-dimensional data-driven tuning. First, we refine the learning-theoretic upper bound using real algebraic geometry; by analyzing invariant connected sign cells during block elimination rather than isolated sign vectors, we avoid topological over-counting to derive strictly sharper sample complexities. Second, we present a multi-regime lower-bound framework that disentangles combinatorial and algebraic capacities. By constructing shattered problem instances across distinct regimes, we prove our upper bounds are tightly saturated. Finally, we extend our topological framework to accommodate general bi-level validation-loss tuning and broader semi-algebraic applications.
\end{abstract}


\section{Introduction}
The success of modern machine learning relies heavily on the meticulous selection of hyperparameters~\cite{probst2018tunability, a2024hyperparameter}. Despite its central role in model deployment, hyperparameter tuning is predominantly treated as an empirical art rather than a rigorous analytical discipline. Traditional approaches, such as exhaustive grid search or random search, discretize continuous parameter spaces to identify high-performing configurations through empirical validation. While straightforward in practice, these brute-force methods are theoretically unprincipled and fail to provide formal generalization guarantees across continuous hyperparameter spaces.

To automate this process, practitioners often rely on sophisticated heuristic search strategies, including Bayesian optimization \cite{bergstra2011algorithms} and spectral methods \cite{hazan2018hyperparameter}. However, these techniques typically hinge on restrictive structural assumptions, such as modeling the loss landscape as a smooth Gaussian process, that often fail to capture the highly volatile, non-convex nature of modern objective functions. Furthermore, resource-allocation strategies like Hyperband \cite{li2017hyperband} excel at computational efficiency via early stopping, yet they primarily focus on discrete search spaces for fixed problem instances. Consequently, these frameworks lack the theoretical machinery to characterize the fundamental statistical complexity of identifying optimal continuous hyperparameters.

To bridge this theoretical gap, the \textit{data-driven algorithm design} paradigm \cite{gupta2020data, balcan2020data} frames hyperparameter tuning as a formal statistical learning problem over an unknown, application-specific problem distribution $\cD$ over the problem space $\cX$. Given a finite sample of training instances $S \sim \cD^N$, the goal is to identify a hyperparameter configuration that provably generalizes to unseen problem instances $x \sim \cD$ from the exact same distribution. This process is inherently bi-level: the induced loss $\ell_\alpha(x) = \inf_{\theta \in \mathcal{S}(x, \alpha)} g(x, \alpha, \theta)$ is evaluated on a target validation objective $g$, subject to the model parameters $\theta$ being optimized over a surrogate training objective $f$, that is, $\mathcal{S}(x, \alpha) = \arg\min_{\theta \in \Theta} f(x, \alpha, \theta)$. For example, in ridge regression, a problem instance $x = (A, b, A', b') \in \mathcal{X}$ consists of a training and a validation set. The hyperparameter $\alpha$ explicitly regularizes the training objective $f(x, \alpha, \theta) = \Vert{}A\theta - b\Vert{}_2^2 + \alpha\Vert{}\theta\Vert{}_2^2$, but impacts the validation target $g(x, \alpha, \theta) = \Vert{}A'\theta - b'\Vert{}_2^2$ only implicitly via the optimal weights $\theta \in \mathcal{S}(x, \alpha)$. Bounding the statistical complexity of this non-smooth dependency, where $\alpha$ acts exclusively through the argmin of an auxiliary problem, is our fundamental theoretical hurdle.

To overcome this challenge, we leverage statistical learning theory to bound the generalization capacity, measured via the pseudo-dimension, of the hyperparameter-induced loss function class $\mathcal{L} = \{\ell_\alpha : \mathcal{X} \to [-H, H] \mid \alpha \in \mathcal{A}\}$. Our analysis relies on the structural observation that for many practical machine learning problems, the training objective $f_x(\alpha, \theta) \triangleq f(x, \alpha, \theta)$ admits a \textit{piecewise polynomial structure} with respect to both the hyperparameter $\alpha$ and the model parameter $\theta$; see e.g., Definition~\ref{def:piecewise-poly-def} for the formal definition and Figure~\ref{fig:piecewise-poly-parabola} for a simple example. This piecewise polynomial assumption is \textit{highly ubiquitous}: it has been rigorously established across a wide spectrum of applications in classical learning theory~\cite{bartlett1998almost, montufar2014number, bartlett2019nearly} and in recent data-driven algorithm design frameworks~\cite{balcan2021much, balcan2023new, nguyen2026provably}.

Despite the prevalence of this structure, existing generalization guarantees remain fundamentally limited. \citet{balcan2025sample} provided the first formal framework for this setting, but their ad-hoc, low-dimensional geometric analysis is highly restrictive: it applies exclusively to \textit{one-dimensional hyperparameters} ($\alpha \in \mathbb{R}$) and \textit{single-level objectives} ($f \equiv g$). \citet{le2026provably} successfully extended this analysis to multi-dimensional hyperparameters ($\alpha \in \mathbb{R}^p$) and general bi-level objectives ($f \not\equiv g$) using a model-theoretic approach. However, their statistical bounds are demonstrably suboptimal. Relying heavily on Quantifier Elimination (QE) \cite{basu2006algorithms} followed by the Goldberg-Jerrum (GJ) framework \cite{bartlett2022generalization}, their approach generates overly conservative upper bounds due to topological under-counting and inflated algebraic dependencies. Furthermore, their approach fails to establish a lower bound that captures the true combinatorial capacity of the problem, specifically one that scales with the piecewise structural complexities, such as the number of pieces and the number of boundaries. These unresolved theoretical gaps directly motivate the tight bounding techniques developed in this work.

\noindent\textbf{Contributions.} In this work, we resolve the theoretical limitations of prior data-driven hyperparameter tuning frameworks by establishing improved pseudo-dimension bounds for the multi-dimensional setting. Our core contributions are:

\begin{itemize}[leftmargin=*]
    \item We establish in Theorem~\ref{thm:general-pdim-upper-bound-via-block-eliminiation} that if a loss function is definable in polynomial \textit{first-order logic} (FOL), its pseudo-dimension is tightly bounded by the topological complexity of a nested block elimination process. By analyzing invariant connected sign cells, this approach bypasses the inflated algebraic dependencies of prior quantifier elimination methods \citep{le2026provably}.
    
    \item We apply our nested block elimination framework to the standard training-loss setting ($f \equiv g$), where the underlying objective $f(x, \alpha, \theta)$ exhibits a piecewise polynomial structure (Theorem 5.1). By representing the induced loss $\ell_\alpha(x) = \min_{\theta \in \Theta} f(x, \alpha, \theta)$ as a polynomial FOL, we derive strictly sharper pseudo-dimension bounds that directly resolve the suboptimal multi-dimensional guarantees of prior work \citep{le2026provably}.

    \item We complement our upper bounds with a novel multi-regime lower-bound framework in Lemmas~\ref{lm:lower-bound-tf} and~\ref{lm:lower-bound-mf}. By independently analyzing the combinatorial ($T_f, M_f$) and algebraic ($\Delta_f$) capacities, we prove that our upper bounds are tight with respect to $T_f$ and $\Delta_f$, and nearly tight with respect to $M_f$. This establishes the fundamental necessity of each structural parameter in our theoretical guarantees.

    \item We extend our framework to the general bi-level validation-loss setting and broader semi-algebraic applications (Theorems~\ref{thm:validation-loss-upper-bound} and \ref{thm:data-driven-weighted-group-lasso}). By applying nested block elimination to validation-loss tuning, we eliminate the inflated algebraic dependencies of prior work. Furthermore, we demonstrate the framework's versatility by analyzing the Weighted Group Lasso, accommodating structures beyond piecewise polynomials and strictly improving existing sample complexity bounds by a full factor of $p$.
\end{itemize}

\noindent \textbf{Technical difference and overviews.} Bounding the generalization capacity of multi-dimensional, bi-level hyperparameter tuning requires capturing the implicit, non-smooth dependency of the validation loss on the hyperparameters. Prior model-theoretic work approached this by applying standard Quantifier Elimination (QE) followed by the Goldberg-Jerrum framework~\citep{le2026provably}. However, this strategy yields highly suboptimal statistical bounds due to severe topological overcounting.  To resolve this, we introduce a novel geometric framework based on nested block elimination. Rather than evaluating isolated sign conditions, we recursively construct a nested block-sign profile that tracks the logical invariance of polynomial formulas across entire connected sign cells. This topological shift completely bypasses the algebraic inflation inherent to standard QE, enabling us to derive strictly sharper, optimal sample-complexity guarantees for the multi-dimensional tuning problem. See Appendix~\ref{apx:detailed-technique-comparision} for a detailed discussion.
 
\noindent \textbf{Notations.} For a real-valued $t \in \bbR$, we denote $\sign(t) = 0$ if $t = 0$, $\sign(t) = 1$ if $t > 0$, and $\sign(t) = -1$ otherwise. For a \textit{variable} $z \in \bbR^k$, we denote $\bbR[z]$ the \textit{polynomial ring} of $z$, containing all (multivariate) polynomials of $z$; e.g., $P(z) = z_1^2 + z_2^2 \in \bbR[z]$ for $z \in \bbR^2$. Given a polynomial $P(z)$, we denote $\deg(P)$ the \textit{degree} of $P(z)$; e.g., $\deg(P) = 2$ for $P(z) = z_1^2 + z_2^2$. For a function $h(z, x)$ that takes both variables $z$ and $x$ as inputs, we denote $h_x(z) \triangleq h(z, x)$ the \textit{induced function} when we treat $x$ as fixed and vary $z$. For a logic sentence $A$, we denote $\bbI(A) = 1$ if $A$ is true, else 0; e.g., $\bbI(1 > 0) = 1$, and $\bbI(0 > 1) = 0$.

\subsection{Related Works}
\paragraph{Data-driven algorithm design.} Data-driven algorithm design \cite{gupta2020data, balcan2020data} adapts hyperparameters to specific problem distributions rather than relying on worst-case instances. This paradigm has achieved significant empirical success across diverse domains, including sketching \cite{li2023learning, indyk2019learning}, linear and mixed-integer programming \citep{sakaue2024generalization, nguyen2026provably, balcan2022structural, cheng2026generalization, le2026provablymilp}, and regularization tuning \cite{balcan2022provably, balcan2023new}.

\paragraph{Theoretical guarantees for data-driven algorithm design.} Motivated by empirical successes, recent work has sought to establish statistical guarantees for data-driven algorithm design \citep{balcan2021much, bartlett2022generalization, balcan2025algorithm}. However, the non-smooth, potentially discontinuous hyperparameter landscape $\ell_x(\alpha)$ makes this challenging, leading most analyses to avoid inner optimization over model parameters $\theta$. While \citet{balcan2025sample} tackled this harder bi-level case, their framework is restricted to one-dimensional hyperparameters ($\alpha \in \mathbb{R}$) and single-level objectives ($f \equiv g$). \citet{le2026provably} recently extended this to multi-dimensional ($\alpha \in \mathbb{R}^p$) and general bi-level settings ($f \not\equiv g$), but their complexity bounds remain theoretically loose and lack lower bounds; these are the gaps we resolve in this work.

\section{Preliminaries}
\subsection{Backgrounds on Learning Theory}
We first recall some standard results in learning theory, which play the central role in this work. 

\begin{definition}[\textbf{Pseudo-dimension} \cite{pollard1984convergence}] \label{def:pseudo-dimension}
    Consider a real-valued function class $\cL = \{\ell_{\alpha}: \cX \rightarrow \bbR \mid \alpha \in \cA\}$ parameterized by $\alpha \in \cA$. Given a set of inputs $S = (x_1, \dots, x_N) \subset \cX$, we say that $S$ is shattered by $\cL$ if there exists a set of real-valued threshold $\tau_1, \dots, \tau_N \in \bbR$ such that $\abs{\{(\bbI(\ell_{\alpha}(x_1)  \geq \tau_1), \dots, \bbI(\ell_{\alpha}(x_N)  \geq \tau_N)) \mid \ell_{\alpha} \in \cL\}}= 2^N$. The pseudo-dimension of $\cL$, denoted $\Pdim(\cL)$, is the maximum size $N$ of an input set that $\cL$ can shatter.
\end{definition}

A finite pseudo-dimension guarantees uniform convergence via empirical risk minimization (ERM).

\begin{theorem}[\citet{pollard1984convergence}] \label{thm:pdim-to-generalization}
    Consider a real-valued function class $\cL = \{\ell_{\alpha}: \cX \rightarrow [-H, H] \mid \alpha \in \cA\}$ parameterized by $\alpha \in \cA$. Assume that $\Pdim(\cL)$ is finite. Then given $\epsilon > 0$ and $\delta \in (0, 1)$, for any $N \geq N(\epsilon, \delta)$, where $N(\epsilon, \delta) = \cO\left(\frac{H^2}{\epsilon^2}(\Pdim(\cL) + \log(1/\delta))\right)$, with probability at least $1 - \delta$ over the draw of $S = (x_1, \dots, x_N) \sim \cD^N$, where $\cD$ is a distribution over $\cX$, we have $\bbE_{x \sim \cD}[\ell_{\hat{\alpha}}(x)] \leq \inf_{\alpha \in \cA} \bbE_{x \sim \cD}[\ell_{\alpha}(x)] + \epsilon.$ Here $\hat{\alpha} \in \arg \min_{\alpha \in \cA} \sum_{x \in S} \ell_{\alpha}(x)$ is the ERM minimizer w.r.t.~the set of instances $S$.
\end{theorem}

\subsection{Backgrounds on Real-Algebraic Geometry}
In this section, we present the necessary background for our block elimination approach. First, we introduce the notion of \textit{first-order formula} (FOL).

\begin{definition}[{First-order formula, quantified/free variables, and polynomial first-order logic} \cite{renegar1992computational}] \label{def:first-order-formula}
    A first-order formula (FOL) $\Phi(\alpha)$ admits the form:
    \begin{equation} \label{eq:FOL}
        (q_1 \theta^{[1]} \in \bbR^{d_1}) \ldots (q_K \theta^{[K]} \in \bbR^{d_K}) \psi(\alpha, \theta^{[1]}, \ldots, \theta^{[K]}),
    \end{equation}
    where
    \begin{enumerate}[leftmargin=*]
        \item Each $q_k$ is one of the quantifiers $\exists$ or $\forall$. The sequence $\{q_k\}_{k = 1}^K$ alternates between $\exists$ and $\forall$ and we denote $K$ as the number of quantifier blocks.
        
        \item $\theta^{[1]}, \theta^{[2]}, \ldots, \theta^{[K]}$ are called the quantified variables, while $\alpha \in \bbR^p$ is called the free variable.
        
        \item $\psi(\alpha, \theta^{[1]}, \theta^{[2]}, \ldots, \theta^{[K]})$ is a boolean combination of atomic predicates of the form:
        \begin{equation*}
            P_j(\alpha, \theta^{[1]}, \theta^{[2]}, \ldots, \theta^{[K]}) \, \chi_j \, 0,
        \end{equation*}
        where $\chi_j \in \{>, \geq, <, \leq, =, \neq\}$ is relational operator.
    \end{enumerate}
    The formula $\psi$ is called the quantifier-free part of $\Phi$. A FOL $\Phi$ is a \emph{polynomial} FOL if each $P_j$ is a polynomial of $\alpha$ and $\theta^{[1]}, \dots, \theta^{[K]}$. 
\end{definition}

While our bi-level optimization naturally maps to a polynomial FOL, applying standard QE yields conservative upper bounds. To derive a sharper sample complexity, we track logical invariance across geometric regions rather than isolated points. We formalize these regions using sign conditions.

\begin{definition}[Sign conditions and their realizations]
    Let $\cP = \{P_1, \dots, P_s\} \subset \bbR[z]$ be a finite set of polynomials in $z \in \bbR^p$, and let $Z \subset \bbR^p$. A sign condition on $\cP$ is a vector $\sigma = (\sigma_1, \dots, \sigma_s) \in \{-1, 0, 1\}^s$. The realization $\Reali_{\cP}(\sigma, Z)$ of sign condition $\sigma$ over $Z$ is a set of $z \in Z$ such that $\sign(P_i(z)) = \sigma_i$ for every $i = 1, \dots, s$, that is,
    \[
        \Reali_\cP(\sigma, Z) = \{z \in Z \mid \sign(P_i(z)) = \sigma_i, i = 1, \dots, s\}.
    \]
    We then denote $\Sign(\cP, Z)$ the set of all possible sign conditions $\sigma$ that we can have by varying $z \in Z$, that is
    \[
        \Sign(\cP, Z) = \{\sigma \in \{-1, 0, 1\}^s \mid \Reali_\cP(\sigma, Z) \neq \emptyset\}.
    \]
    Equivalently, $\Sign(\cP, Z) = \{\sign(\cP(z)): z \in Z\}$, where we abuse notation and write
    \[
        \sign(\cP(z)) = (\sign(P_1(z)), \dots, \sign(P_s(z)).
    \]
    When $Z = \bbR^p$, we shorten the notation and write $\Reali_\cP(\sigma) = \Reali_\cP(\sigma, \bbR^p)$ and $\Sign(\cP, \bbR^p) = \Sign(\cP)$.
\end{definition}

\begin{definition}[Connected sign cells] \label{def:connected-sign-cell}
    Let $\cP \subset \bbR[z]$ be a finite set of polynomials. For each realizable sign pattern $\sigma \in \Sign(\cP)$, let $\Cc(\Reali_\cP(\sigma))$ denote the set of connected components of the realization $\Reali_\cP(\sigma)$. We then denote $\Cell(\cP)$ the set of connected sign cells induced by $\cP$, that is
    \[
        \Cell(\cP) = \bigcup_{\sigma \in \Sign(\cP)} 
        \Cc(\Reali_\cP(\sigma)).
    \]
    We further denote $b_0(S)$ the number of connected components of $S$, and therefore
    \[
        |\Cell(\cP)| = \sum_{\sigma \in \Sign(\cP)}b_0(\Reali_\cP(\sigma)).
    \]
\end{definition}

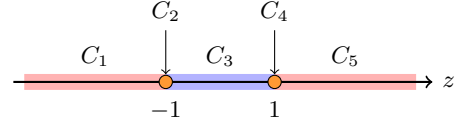
\begin{figure}[t]
    \centering

    \begin{tikzpicture}[
        x=0.72cm,
        y=0.72cm,
        font=\footnotesize
    ]
        \draw[
            line width=6pt,
            red!30,
            line cap=butt
        ] (-3.6,0) -- (-1,0);

        \draw[
            line width=6pt,
            blue!30,
            line cap=butt
        ] (-1,0) -- (1,0);

        \draw[
            line width=6pt,
            red!30,
            line cap=butt
        ] (1,0) -- (3.6,0);

        \draw[->,thick]
        (-3.8,0) -- (3.9,0)
        node[right] {$z$};

        \filldraw[
            fill=orange!80,
            draw=black,
            line width=0.4pt
        ] (-1,0) circle (2.5pt);

        \filldraw[
            fill=orange!80,
            draw=black,
            line width=0.4pt
        ] (1,0) circle (2.5pt);

        \node[below=4pt] at (-1,0) {$-1$};
        \node[below=4pt] at (1,0) {$1$};

        \node at (-2.3,0.47) {$C_1$};
        \node at (0,0.47) {$C_3$};
        \node at (2.3,0.47) {$C_5$};

        \draw[->,thin]
        (-1,0.95) node[above] {$C_2$}
        -- (-1,0.14);

        \draw[->,thin]
        (1,0.95) node[above] {$C_4$}
        -- (1,0.14);

    \end{tikzpicture}

    \caption{
        A demonstration of sign conditions, realizations, and cells for $\cP = \{P(z) = z^2 - 1\} \subset \bbR[z]$. The realization $\Reali_\cP(1)$ yields two connected components: $C_1 = (-\infty, -1)$ and $C_5 = (1, \infty)$. Conversely, $\Reali_\cP(-1)$ yields a single component $C_3 = (-1, 1)$, and $\Reali_\cP(0)$ yields $C_2 = \{-1\}$ and $C_4 = \{1\}$. Ultimately, $\cP$ induces $\vert{}\Cell(\cP)\vert{} = 5$ connected sign cells: $\Cell(\cP) = \{C_1, C_2, C_3, C_4, C_5\}$.
    }
    \label{fig:connected-sign-cells}
\end{figure}

Figure~\ref{fig:connected-sign-cells} demonstrates the concepts above. We now formulate block elimination in terms of these connected sign cells. Given polynomials in free variables $z$ and variables $t$ to be eliminated, the block elimination process aims to produce a set of polynomials in~$z$ whose connected sign cells preserve the sign information attainable by varying $t$. Critically, our multi-dimensional tuning problem involves distinct blocks of quantified variables that must be eliminated sequentially. Evaluating them simultaneously destroys this block-wise grouping, leading to topological under-counting. To bypass this algebraic inflation, we track sign invariance sequentially by constructing nested profiles.

\begin{definition}[Nested block-sign profiles] \label{def:nested-block-sign-profiles}
    Let $\cP = \{P_1, \dots, P_s\} \subset \bbR[z, t^{(1)}, \dots, t^{(K)}]$, where $z \in \bbR^p$ and $t^{(k)} \in \bbR^{d_k}$ for $k = 1, \dots, K$. We define the terminal sign vector $\Sigma_K(z, t^{(1)}, \dots, t^{(K)}) = \sign(\cP(z, t^{(1)}, \dots, t^{(K)}))$. Recursively, for $k = K, \dots, 1$, we define
    \begin{align*}
        &\Sigma_{k - 1}(z, t^{(1)}, \dots, t^{(k - 1)})\!=\! \\
        & \qquad \qquad  \{\Sigma_k(z, t^{(1)}, \dots, t^{(k - 1)}, u)\!:\!u \in \bbR^{d_k}\}.
    \end{align*}
    Then, we call $\Sigma_0(z)$ the nested block-sign profile of $\cP$ at a fixed point $z \in \bbR^p$, relative to the ordered variable blocks $t^{(1)}, \dots, t^{(K)}$.
\end{definition}

Roughly speaking, in Definition~\ref{def:nested-block-sign-profiles}, we can think of the \textit{terminal sign vector} $\Sigma_K$ as an ordinary sign vector, and subsequently $\Sigma_{K - 1}$ is a set of sign vectors, and so on. Finally, the nested block-sign profile $\Sigma_0(z)$ can be thought of as a depth-$K$ sign tree. Specifically, unlike the ordinary set of sign conditions obtained by varying all quantified variables simultaneously, $\Sigma_0(z)$ retains the block-wise grouping of realizable sign patterns, which is crucial in our analysis. Based on this definition, we can now formally define the \textit{nested sign-invariant projection} as follows.

\begin{definition}[Nested sign-invariant projection] \label{def:nested-sign-invariant-projection}
    A set of polynomials $Q \subset \bbR[z]$ is a nested sign-invariant projection of $\cP \subset \bbR[z, t^{(1)}, \dots, t^{(K)}]$, relative to the ordered blocks $t^{(1)}, \dots, t^{(K)}$, if the nested block-sign profile $\Sigma_0(z)$ is constant on every connected sign cell induced by $\cQ$. That is, for every $C \in \Cell(\cQ)$ and every $z_1, z_2 \in C$, we have $\Sigma_0(z_1) = \Sigma_0(z_2)$.
\end{definition}

Specially, when $K = 1$, we have $\Sigma_0(z) = \{\sign(\cP(z, u): u \in \bbR^{d_1}\} = \Sign(\cP_z, \bbR^{d_1})$. Thus, in the one-block case, Definition~\ref{def:nested-sign-invariant-projection} simply requires the set of realizable sign conditions to be unchanged on every connected sign cell induced by $\cQ$. To help the audience quickly grasp the concepts of nested block-sign profiles in Definition~\ref{def:nested-block-sign-profiles} and nested-sign invariant projection in Definition~\ref{def:nested-sign-invariant-projection}, we present the following simple example.

\begin{example}
    Let $z, u, v \in \bbR$, with ordered quantified blocks $t^{(1)} = u$ and $t^{(2)} = v$. Consider the set of polynomials $\cP = \{P_1, P_2\} \subset \bbR[z, u, v]$, where $P_1(z, u, v) = u^2 - z$ and $P_2(z, u, v) = v$. Then the terminal sign vector $\Sigma_2(z, u, v)$ is $ \Sigma_2(z, u, v) = (\sign(u^2 - z), \sign(v)).$ For $a \in \{-1, 0, 1\}$, we define $E_a = \{(a, -1), (a, 0), (a, 1)\}$, then after varying the innermost block $v$, we obtain $\Sigma_1(z, u) = E_{\sign(u^2 - z)}$. After subsequently varying $u$, the nested sign profile $\Sigma_0(z)$ becomes
    \[
        \Sigma_0(z) =\begin{cases}
            \{E_1\}, & z< 0, \\
            \{E_0, E_1\}, &z = 0,\\
            \{E_{-1}, E_0, E_1\}, &z > 0.
        \end{cases}
    \]
    Consequently, $\cQ = \{z\}$ is a nested sign-invariant projection of $\cP$ as $\Sigma_0(z)$ is constant on each of the three connected sign cells $(-\infty, 0), \{0\}$, and $(0, \infty)$ induced by $\cQ$. 
\end{example}

To rigorously bound the statistical complexity of bi-level hyperparameter tuning, we must first formalize the implicit dependency in the validation loss. We achieve this by formulating continuous-optimization objectives in polynomial first-order logic (FOL).
 
\begin{theorem}[{{Nested block elimination}, Adapted from \citet[Chapter~14.1]{basu2006algorithms}}] \label{thm:block-elimination}
    Fix $K \geq 1$, let $z \in \bbR^p$, $t^{(k)} \in \bbR^{d_k}$ for $k = 1, \dots, K$, and let $\cP \subset \bbR[z, t^{(1)}, \dots, t^{(K)}]$ be a finite set of polynomials satisfying (i) $|\cP| \leq s$, and (ii) $\deg(R) \leq \Delta$ for every $R \in \cP$. Then there exists a nested sign-invariant projection $\cQ \subset \bbR[z]$ of $\cP$, relative to the ordered blocks $t^{(1)}, \dots, t^{(K)}$, such that the nested block-sign profile $\Sigma_0(z)$ is constant on every $C \in \Cell(\cQ)$. Moreover, defining $A_K \triangleq \prod_{k = 1}^K (d_k + 1), B_K \triangleq \prod_{k = 1}^K d_k,$ then there exists a constant $c_K$, depending only on $K$, such that 
    \[
        |\cQ| \leq s^{A_K}\Delta^{c_KB_K}, \textup{ and }\max_{R \in \cQ}\deg(R) \leq \Delta^{c_KB_K}.
    \]
\end{theorem}

Finally, we recall a useful result that allows us to give an upper bound for the number of elements in the set of connected sign cells $\Cell(\cQ)$ corresponding to a set of polynomials $\cQ$.

\begin{lemma}[{\citet[Theorem~1.1]{basu2005betti}}] \label{lm:cc-upper-bound}
    Let $\cQ \subset \bbR[z]$ be a set of $s$ polynomials in $z \in \bbR^p$ of degree at most $\Delta$, where $s, p, \Delta \geq 1$. Then there exists an universal constant $C$ such that $|\Cell(\cQ)| \leq \left(\frac{Cs\Delta}{p}\right)^p$ if $s \geq p$, and $|\Cell(\cQ)| \leq \left(C\Delta\right)^p$ if $s < p$.
\end{lemma}

\section{Problem Settings}

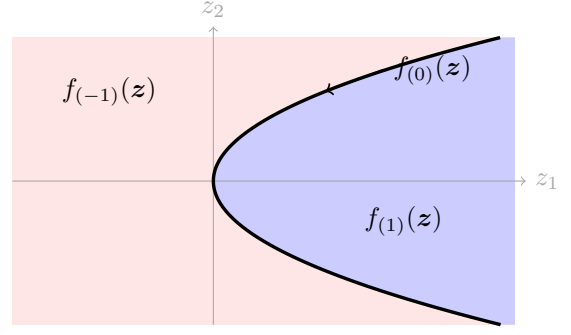
\begin{figure}[t]
    \centering

    \begin{tikzpicture}[x=0.95cm,y=0.95cm]
        \def\xmin{-2.8}
        \def\xmax{4.2}
        \def\ymin{-2.0}
        \def\ymax{2.0}

        \fill[red!10]
            (\xmin,\ymin) rectangle (\xmax,\ymax);

        \fill[blue!20]
            plot[
                domain=\ymin:\ymax,
                samples=120,
                variable=\t
            ]
            ({\t*\t},{\t})
            -- (\xmax,\ymax)
            -- (\xmax,\ymin)
            -- cycle;

        \draw[->,gray!70]
            (\xmin,0) -- (\xmax+0.15,0)
            node[right] {$z_1$};

        \draw[->,gray!70]
            (0,\ymin) -- (0,\ymax+0.15)
            node[above] {$z_2$};

        \draw[line width=0.45mm]
            plot[
                domain=\ymin:\ymax,
                samples=120,
                variable=\t
            ]
            ({\t*\t},{\t});

        \node at (-1.45,1.25)
            {$f_{(-1)}(\boldsymbol{z})$};

        \node at (2.65,-0.55)
            {$f_{(1)}(\boldsymbol{z})$};

        \node (boundarylabel) at (3.05,1.55)
            {$f_{(0)}(\boldsymbol{z})$};

        \draw[->,thick]
            (boundarylabel.west) -- (1.5625,1.25);

    \end{tikzpicture}

    \caption{
        An example of a piecewise polynomial function $f(z)$ (Definition~\ref{def:piecewise-poly-def}) with parabolic boundary $h_1(z)=z_1-z_2^2$. The function evaluates to $f_{(-1)}(z) = z_1$ in the pink region ($\sigma(z)=-1$), $f_{(1)}(z) = z_2$ in the blue region ($\sigma(z)=1$), and $f_{(0)}(z) = z_1^2 + z_2^2$ on the black boundary ($\sigma(z)=0$). With 1 boundary, 3 pieces, and maximum polynomial degree 2, the function's complexity is $(1, 3, 2)$.
    }
    \label{fig:piecewise-poly-parabola}
\end{figure}

Our data-driven framework models hyperparameter tuning across three foundational spaces: a problem instance space $\mathcal{X}\subset\mathbb{R}^{q}$, a continuous hyperparameter space $\mathcal{A}\subset\mathbb{R}^{p}$, and a model parameter space $\Theta\subseteq\mathbb{R}^{d}$.  To ensure analytical compactness, we restrict the parameter domains to bounding boxes, defining $\mathcal{A}=[\alpha_{min},\alpha_{max}]^{p}$ and $\Theta=[\theta_{min},\theta_{max}]^{d}$.  For any given instance $x\in\mathcal{X}$, the learning procedure evaluates a model parameterized by $\theta\in\Theta$ alongside hyperparameters $\alpha\in\mathcal{A}$ via two distinct objectives:  
\begin{itemize}[leftmargin=*]
    \item \textbf{Training Objective} $f:\mathcal{X}\times\mathcal{A}\times\Theta\rightarrow[-H, H]$: The surrogate loss directly minimized by the algorithm during training.
    \item \textbf{Validation Objective} $g:\mathcal{X}\times\mathcal{A}\times\Theta\rightarrow[-H, H]$: The true target metric utilized to assess the quality of the model's out-of-sample performance. 
\end{itemize}
Here, $H$ acts as a positive constant bounding the function values. 

\paragraph{The Bi-level Optimization Objective.} The ultimate performance of a hyperparameter configuration $\alpha$ on a specific problem instance $x$ is quantified by the induced loss $\ell_{\alpha}(x)$. Because hyperparameters generally dictate model behavior implicitly through the training procedure, $\ell_{\alpha}(x)$ is naturally cast as a bi-level optimization problem $\ell_{\alpha}(x)=\inf_{\theta\in\mathcal{S}(x,\alpha)}g(x,\alpha,\theta)$ where the set of optimal lower-level model parameters is defined as $\mathcal{S}(x,\alpha)=\arg\min_{\theta\in\Theta}f(x,\alpha,\theta)$. By evaluating the infimum over the set $\mathcal{S}(x,\alpha)$, this framework employs the standard optimistic bi-level formulation. This guarantees that the validation loss remains strictly well-defined even if the inner surrogate optimization yields a non-singleton set of minimizers.  

\paragraph{Statistical Learning Goal.} We assume the existence of an unknown, application-specific distribution $\mathcal{D}$ from which problem instances are drawn. The theoretical goal of data-driven tuning is to isolate a configuration $\alpha^{*}$ that minimizes the expected validation loss, i.e.,  $\alpha^{*}\in\arg\min_{\alpha\in\mathcal{A}}\mathbb{E}_{x\sim\mathcal{D}}[\ell_{\alpha}(x)]$. Given that the true distribution $\mathcal{D}$ is unobservable, the learner is instead provided with a finite dataset $S=\{x_{1},...,x_{N}\}\sim\mathcal{D}^{N}$ consisting of $N$ training instances. The task therefore reduces to Empirical Risk Minimization (ERM), where the algorithm computes an empirical minimizer $\hat{\alpha}(S)$, that is, $\hat{\alpha}(S)\in\arg\min_{\alpha\in\mathcal{A}}\frac{1}{N}\sum_{i=1}^{N}\ell_{\alpha}(x_{i})$.

\paragraph{Structural Assumptions.} As in prior works \citep{balcan2025sample,le2026provably}, a central theoretical pillar of our subsequent complexity analysis requires the objective functions to exhibit specific algebraic geometries.

\begin{definition}[Piecewise polynomial function {\cite{balcan2025sample}}] 
	\label{def:piecewise-poly-def}
	A real-valued function $f$ admits a \textit{piecewise polynomial structure} with complexity $(M_f, T_f, \Delta_f)$ if there exists a set of boundary polynomials $\mathbb{H} = \{h_1, \dots, h_{M_f}\}$ and a set of value polynomials $\mathbb{F} = \{f_{\boldsymbol{\sigma}}\}_{\boldsymbol{\sigma} \in \Sigma_f} \subset \bbR[z]$, where $\Sigma_f \subseteq \{-1, 0, 1\}^{M_f}$ and $|\Sigma_f| \leq T_f$, both with polynomial degrees at most $\Delta_f$, such that for any $z$, we have $f(z) = f_{\boldsymbol{\sigma}(z)}(z)$, where $\boldsymbol{\sigma}(z) \in \Sigma_{f}$ is the sign pattern of $z$ with respect to the set of functions $\mathbb{H}$, i.e., $\boldsymbol{\sigma}(z)_j = \sign(h_{j}(z))$. The functions in $\mathbb{F}$ are called the \textit{piece functions} and the functions in $\mathbb{H}$ are called the \textit{boundary functions}.
\end{definition}

\begin{assumption}[Structural assumption \cite{le2026provably}] \label{asmp:structural}
    Given any problem instance $x$, the dual parameter-dependent training objective $f_{x}(\alpha, \theta) \triangleq f(x, \alpha, \theta)$ and dual parameter-dependent validation objective $g_{x}(\alpha, \theta) \triangleq g(x, \alpha, \theta)$ admits piecewise polynomial structure as functions of $\alpha$ and $\theta$ with complexity  $(M_f, T_f, \Delta_f)$ and $(M_g, T_g, \Delta_g)$ that are independent of the problem instance $x$, respectively. Moreover, the set $\cS(x,\alpha)$ of minimizes of $f(x,\alpha,\cdot)$ is non-empty, for all $(x,\alpha)$.
\end{assumption}

\begin{remark}[Ubiquity of structure]
    As noted in prior works \cite{balcan2025sample, le2026provably}, this structural assumption is ubiquitous. It has been formally established across a wide range of applications in both classical learning theory \cite{bartlett1998almost, montufar2014number, bartlett2019nearly}) and data-driven algorithm design \cite{balcan2021much, balcan2023new, nguyen2026provably}.
\end{remark}

Although recent work has successfully extended data-driven algorithm design to multi-dimensional hyperparameters and bi-level objectives~\cite{le2026provably}, their sample complexity guarantees are fundamentally suboptimal. By analyzing invariant connected sign cells rather than relying on the inflated algebraic dependencies of Quantifier Elimination, we now present a \textit{nested block elimination framework} that entirely eliminates this theoretical looseness.

\section{A General Learning-theoretic Complexity Framework via Block Elimination} \label{sec:general-framework-main}

In this section, we introduce a new general upper bound for learning-theoretic complexity by leveraging nested block elimination. First, we establish Proposition~\ref{prop:cellwise-logical-invariance}, which guarantees that the truth value of a polynomial first-order logic formula remains invariant across the connected sign cells of a nested sign-invariant projection.

\begin{proposition}[Cellwise logical invariance] \label{prop:cellwise-logical-invariance}
    Let $\Phi(z) = (q_1t^{(1)} \in \bbR^{d^{(1)}})\dots(q_Kt^{(K)} \in \bbR^{d^{(K)}}) \psi(z, t^{(1)}, \dots, t^{(K)})$ be a polynomial FOL, and let $\cP$ be the family of atomic polynomials appearing in the atomic predicates of $\psi$. If $\cQ$ is a nested sign-invariant projection of $\cP$ relative to the ordered blocks $t^{(1)}, \dots, t^{(K)}$, then $\Phi$ has a constant truth value on every connected sign cell $C \in \Cell(\cQ)$. In other words, for every $C \in \Cell(\cQ)$ and every $z_1, z_2 \in C$, we have $\Phi(z_1) \Leftrightarrow \Phi(z_2)$.
\end{proposition}
The proof of Proposition~\ref{prop:cellwise-logical-invariance} is presented in Appendix~\ref{apx:general-framework-main}. Using this cellwise logical invariance alongside the nested block elimination process, we now state the following theorem, which establishes a rigorous pseudo-dimension upper bound for any function class whose threshold conditions can be formulated via branches of polynomial first-order logic.

\begin{theorem}[Pseudo-dimension upper-bound via branchwise nested block elimination] \label{thm:general-pdim-upper-bound-via-block-eliminiation}
    Consider a real-valued function class $\cF = \{f_\alpha: \cX \rightarrow \bbR \mid \alpha \in \cA\}$ be a real-valued function class, where $\cA = [\alpha_{\min}, \alpha_{\max}]^p \subseteq \bbR^p$. Suppose that there are integers $L, M, K \geq 1$, where $K$ is fixed, block dimensions $d_1, \dots, d_K \geq 1$, and a degree upper bound $\Delta \geq 1$ satisfying the following conditions: for any $x \in \cX$ and a real-valued threshold $\tau \in \bbR$, there exists
    \begin{enumerate}
        \item $L$ polynomial FOLs $\Phi_{x, \tau, 1}(\alpha), \dots, \Phi_{x, \tau, L}(\alpha)$, and
        \item a Boolean function $\beta_{x, \tau}: \{0, 1\}^L \rightarrow \{0, 1\}$
    \end{enumerate}
    such that $\bbI(f_\alpha(x) \geq \tau) = \beta_{x, \tau}(\Phi_{x, \tau, 1}(\alpha), \dots, \Phi_{x, \tau, L}(\alpha))$ for every $\alpha \in \cA$. Furthermore, assuming that every FOL $\Phi_{x, \tau, \ell}$:
    \begin{enumerate}
        \item has $K$ ordered quantified variables blocks of dimension $d_1, \dots, d_K$, respectively, and 
        \item contains at most $M$ distinct atomic polynomials, each of degree at most $\Delta$.
    \end{enumerate} 
     Define $A_K = \prod_{k = 1}^K(d_k + 1)$ and $B_K = \prod_{k = 1}^K d_k$, then the pseudo-dimension of the function class $\cF$ is at most 
     \[
        \Pdim(\cF) = \cO\left(p\log(2L) + pA_K \log (2M) + pB_K\log (2\Delta)\right).
     \]
\end{theorem}

\noindent \textit{Proof Sketch.} 
    Assume there exists $N$ problem instances $x_1, \dots, x_N$ that can be shattered by $\cF$, with real-valued thresholds $\tau_1, \dots, \tau_N$ witness the shattering. Let $S_K = M^{A_K}\Delta^{c_KB_K}$ and $\Delta_K = \Delta^{c_KB_K}$, where $c_K$ depends only on $K$ from Theorem~\ref{thm:block-elimination}. 
    
    For every instance $i = 1, \dots, N$, and a branch $\ell = 1, \dots, L$, let $\cP_{i, \ell}$ be the set of atomic polynomials appearing in $\Phi_{x_i, \tau_i, \ell}$. From Theorem~\ref{thm:block-elimination}, there exists a nested sign-invariant projection $\cQ_{i, \ell}\subset \bbR[\alpha]$ such that: (1) $|\cQ_{i, \ell}| \leq S_K$, and (2) $\max_{R \in \cQ_{i, \ell}}\deg (R) \leq D_K$. By Proposition~\ref{prop:cellwise-logical-invariance}, the truth value of $\Phi_{x_i, \tau_i, \ell}$ remains unchanged on every connected sign cell induced by $\cQ_{\ell_i}$. 
    
    Now, let $\cB_\cA$ be a family of $2p$ affine polynomials describing the region $\cA$, and let $\Tilde{\cQ} = \cB_\cA\cup \bigcup_{i = 1}^N\bigcup_{\ell = 1}^L \cQ_{i, \ell}$. By definition, we have (1) $p \leq \abs{\Tilde{\cQ}} \leq NLS_K + 2p$, and (2) $\max_{R \in \Tilde{\cQ}} \deg(R) \leq D_K$. Furthermore, from the previous observation about $\cQ_{i, \ell}$, we claim that over every connected sign cell $C \in \Cell(\tilde{\cQ})$, the threshold label 
    \[
        \bbI(f_\alpha(x) \geq \tau) = \beta_{x, \tau}(\Phi_{x, \tau, 1}(\alpha), \dots, \Phi_{x, \tau, L}(\alpha))
    \]
    remains unchanged. Therefore, we claim that
    \[
        2^N \leq |\Cell(\tilde{\cQ})| \leq \left(\frac{\kappa (NLS_K + 2p) D_K}{p}\right)^p,
    \]
    where the latter inequality comes from Lemma~\ref{lm:cc-upper-bound}. Finally, solving this inequality gives the final conclusion. See Appendix~\ref{apx:general-framework-main} for the detailed proof. 
\qed

\section{Tuning via Training Objective} \label{sec:training-loss}
In this section, we focus exclusively on the standard training-loss setting, where the training and validation objectives are identical (i.e., $f \equiv g$). Under this regime, we apply our nested block elimination framework to the underlying piecewise polynomial objective to derive a sharp pseudo-dimension upper bound in Section~\ref{sec:training-loss-upper-bound}, followed by a comprehensive lower bound analysis in Section~\ref{sec:training-loss-lower-bound} to demonstrate tightness.

\subsection{Upper Bound} \label{sec:training-loss-upper-bound}
We first bound the generalization capacity of the hyperparameter-induced loss function. By formulating the inner minimization problem as a polynomial first-order logic sentence, we can directly invoke the block elimination argument (Theorem~\ref{thm:block-elimination}) to derive the following upper bound.  

\begin{theorem}[Pseudo-dimension] \label{thm:training-loss-upper-bound}
    Let $\cA = [\alpha_{\min}, \alpha_{\max}]^p$ and $\Theta = [\theta_{\min}, \theta_{\max}]^d$ be the hyperparameter and parameter domains, respectively. Assume that for any problem instance $x \in \cX$, the training objective $f_x(\alpha, \theta)\triangleq f(x, \alpha, \theta)$, for any $(\alpha, \theta) \in \cA \times \Theta$, admits a piecewise polynomial structure with complexity $(M_f, T_f, \Delta_f)$. Then, the pseudo-dimension of the function class $\cL = \{\ell_\alpha: \cX \rightarrow \bbR \mid \alpha \in \cA\}$, where $\ell_\alpha(x) = \min_{\theta \in \Theta}f(x, \alpha, \theta)$, is at most
    \[
        \Pdim(\cL) = \cO(p\log T_f + pd\log (M_f + d) + pd \log \Delta_f). 
    \]
\end{theorem}
\noindent \textit{Proof sketch.}
     For any problem instance $x\in\mathcal{X}$ and threshold $\tau\in\mathbb{R}$, our goal is to express the boolean indicator $\mathbb{I}(\ell_{x}(\alpha)\ge\tau)$ as a boolean combination of polynomial first-order logic (FOL) formulas to directly invoke Theorem~\ref{thm:block-elimination}. Because the induced loss is defined via inner minimization, $\ell_{x}(\alpha)=\min_{\theta\in\Theta}f_{x}(\alpha,\theta)$, the condition $\ell_{x}(\alpha)\ge\tau$ is logically equivalent to the universal statement $(\forall\theta\in\Theta)[f_{x}(\alpha,\theta)\ge\tau]$. Exploiting the piecewise polynomial structure of $f_x$, we partition the parameter space into invariant sign regions. Let $\textup{Where}_{x,\sigma}(\alpha,\theta) \triangleq \bigwedge_{m=1}^{M_{x}}[\sign(h_{x,m}(\alpha,\theta))=\sigma_{m}]$ be the predicate indicating whether $(\alpha,\theta)$ falls into the region defined by the sign pattern $\sigma \in \Sigma_x$. We can rewrite $\mathbb{I}(\ell_{x}(\alpha)\ge\tau)$ as a conjunction over all valid regions $\bigwedge_{\sigma\in\Sigma_{x}} \Phi_{x,\tau,\sigma}(\alpha)$ where
     $$
        \Phi_{x,\tau,\sigma}(\alpha) \triangleq (\forall\theta\in\Theta)[\textup{Where}_{x,\sigma}(\alpha,\theta)\Rightarrow P_{x,\sigma}(\alpha,\theta)\ge\tau].
     $$
     By substituting the domain constraint $\textup{In}_{\Theta}(\theta) = \bigwedge_{r=1}^{d}(\theta_{\max}-\theta_r\ge0)\wedge(\theta_r-\theta_{\min}\ge 0)$ and applying the logical identity $(A\Rightarrow B) \equiv (\neg A\vee B)$, we transform each branch into a standard polynomial FOL $\Phi_{x,\tau,\sigma}(\alpha) = $
     $$  
        (\forall\theta\in\mathbb{R}^{d})[\neg \textup{In}_{\Theta}(\theta)\vee\neg \textup{Where}_{x,\sigma}(\alpha,\theta)\vee P_{x,\sigma}(\alpha,\theta)\ge\tau].
    $$
     
     Crucially, each formula $\Phi_{x,\tau,\sigma}$ requires exactly $K=1$ quantifier block of dimension $d_1=d$, and contains at most $M_{f}+2d+1$ distinct atomic polynomial predicates (accounting for the active boundary pieces, domain edges, and the threshold comparison). Furthermore, the maximum degree across all atomic predicates is strictly bounded by $\Delta_{f}$. Substituting these uniform complexity parameters into our nested block elimination bound in Theorem~\ref{thm:general-pdim-upper-bound-via-block-eliminiation} yields the final pseudo-dimension. The detailed proof can be found in Appendix \ref{apx:training-loss-upper-bound}.
\qed

\begin{remark}[Comparison to prior works]
    Theorem~\ref{thm:training-loss-upper-bound} strictly improves upon the multi-dimensional generalization guarantees established in prior works in \citet[Theorem 5.1]{le2026provably}. Their approach, which relied on standard quantifier elimination, yielded a suboptimal upper bound of $Pdim(\mathcal{L}) = \mathcal{O}(pd\log(M_f + T_f + d) + p^2 d \log \Delta_f)$. By utilizing nested block elimination, we remove the inflated $p^2$ dependency on the algebraic capacity and significantly sharpen the combinatorial dependencies. Furthermore, while prior work only established a partial lower bound of $\Omega(pd \log \Delta_f)$, we resolve this gap in Section 5.2 by providing a complete theoretical lower bound that explicitly captures the combinatorial complexities $M_f$ and $T_f$. We provided more in-depth discussion in Appendix~\ref{apx:training-loss-upper-bound}. 
\end{remark}

\subsection{Lower Bound} \label{sec:training-loss-lower-bound}
As noted above, prior work only provided an algebraic lower bound of $\Omega(pd \log \Delta_f)$ in \citet[Theorem~5.2]{le2026provably}, leaving the necessity of the combinatorial parameters $T_f$ and $M_f$ an \textit{open question}. We address this gap by introducing a multi-regime lower bound framework. By independently analyzing these capacities, we prove that our upper bounds are tight with respect to the number of pieces $T_f$ and nearly tight with respect to the number of boundaries $M_f$, establishing the fundamental necessity of each structural parameter in our theoretical guarantees. The detailed proofs of Lemmas~\ref{lm:lower-bound-tf} and~\ref{lm:lower-bound-mf} can be found in Appendix \ref{apx:training-loss-lower-bound}.

\begin{lemma}[Lower bound depending on $T_f$] \label{lm:lower-bound-tf}
    For every $p, d \geq 1$ and an integer $T_f \geq 4$, there exists box like regions $\cA \subset \bbR^p$, $\Theta \subset \bbR^d$, an instance space $\cX$, and a training objective $f(x, \alpha, \theta)$ that admits piecewise polynomial structure with complexity $(\lfloor \frac{T_f - 1}{2}\rfloor, T_f, 1)$ such that the corresponding function class $\cL$ satisfies: $\Pdim(\cL) = \Omega\left(p \log T_f\right).$
\end{lemma}

\begin{lemma}[Lower bound depending on $M_f$] \label{lm:lower-bound-mf}
    For every $p, d, M_f \geq 1$, there exists box-like hyperparameter domain $\cA \subset \bbR^p$ and parameter domain $\Theta \subseteq \bbR^d$, a finite instance space $\cX$, and a training objective $f$ such that:
    \begin{enumerate}
        \item[(i)] For any $x \in \cX$, the function $f_x(\alpha, \theta) \triangleq f(x, \alpha, \theta)$ admits a piecewise polynomial structure with exacts $M_f$ affine boundary polynomials, at most $\left(1 + \frac{2M_f}{d}\right)^d$ realized sign conditions, piece polynomials of degree at most $2$.
        \item[(ii)] The induced class $\cL = \{\ell_\alpha: \cX \rightarrow \bbR \mid \alpha \in \cA\}, \ell_\alpha(x) = \min_{\theta \in \Theta} f(x, \alpha, \theta)$ satisfies $\Pdim(\cL) = \Omega\left(pd \log\left(1 + \frac{M_f}{d}\right)\right)$.
    \end{enumerate}
\end{lemma}

\begin{remark}[Tightness of the bounds]
     Our multi-regime lower bounds, combined with prior results, demonstrate that the pseudo-dimension upper bound in Theorem 5.1 is tight in practically relevant regimes. Specifically, Lemma 5.2 yields a lower bound of $\Omega(p \log T_f)$, and \citet[Theorem 5.2]{le2026provably} establishes an algebraic lower bound of $\Omega(pd \log \Delta_f)$. These match the respective dependencies in Theorem 5.1, proving our bound is strictly tight with respect to both $T_f$ and $\Delta_f$. Furthermore, Lemma 5.3 isolates the boundary capacity to establish a lower bound of $\Omega(pd \log(1 + \frac{M_f}{d}))$, confirming that the theoretical dependency on $M_f$ is nearly tight. 
\end{remark}

\section{Tuning via Validation Objective} \label{sec:validation-loss}

We now consider the general bi-level setting in which model parameters are selected using the training objective $f$, while the hyperparameters are evaluated using a potentially different validation objective $g$. Recall that  $\ell_{\alpha}^{\text{val}}(x) = \inf_{\theta \in \mathcal{S}(x,\alpha)} g_x(\alpha,\theta)$, where $\mathcal{S}(x,\alpha) = \arg\min_{\theta \in \Theta} f_x(\alpha,\theta)$. Unlike the training-loss setting of Section~\ref{sec:training-loss}, analyzing this loss requires encoding validation performance and lower-level optimality simultaneously.

\begin{theorem} \label{thm:validation-loss-upper-bound}
    Under Assumption 1, let $\mathcal{L}_{\text{val}} = \{\ell_{\alpha}^{\text{val}} : \mathcal{X} \to [-H,H] \mid \alpha \in \mathcal{A}\}$, $\Delta_{f,g} = \max\{1, \Delta_f, \Delta_g\}$, and $M_{\textup{tot}} = M_f + M_g + T_f + d$, we have
  \[
    \operatorname{Pdim}(\mathcal{L}_{\text{val}}) = \mathcal{O}\left(p \log(T_f T_g) + p d^2 \log\left(2 + M_{\textup{tot}}\Delta_{f,g}\right)\right).
\]
\end{theorem}
The detailed proof can be found in Appendix~\ref{apx:validation-loss}. 
Compared with prior multi-dimensional validation bounds in \citet[Theorem~6.1]{le2026provably}, Theorem~\ref{thm:validation-loss-upper-bound} is strictly tighter by removing one factor $p$ from the degree-dependent term and moves the $T_g$ dependence outside the $d^2$-scaled logarithm. See Appendix~\ref{apx:validation-loss} for a more detailed comparison.

\section{Applications} \label{sec:applications}
To demonstrate the versatility of our nested block elimination framework, we analyze the data-driven Weighted Group Lasso \cite{yuan2006model, le2026provably}. The training and validation objectives are formulated as: $f(x, \alpha, \theta) = \Vert{}A\theta - b\Vert{}_2^2 + \sum_{i=1}^p \alpha_i \Vert{}\theta_i\Vert{}_2$, and $g(x, \alpha, \theta) = \frac{1}{2}\|A'\theta - b'\|^2$. Here, the model parameter $\theta$ is partitioned into $p$ distinct groups $\theta = (\theta_1, \dots, \theta_p) \in \mathbb{R}^{d_1} \times \dots \times \mathbb{R}^{d_p}$ where $\sum_{i=1}^p d_i = d$, and the hyperparameter $\alpha \in \mathbb{R}^p$ controls the penalty weight for each group. Using our proposed framework, we have the following results. 

\begin{theorem} \label{thm:data-driven-weighted-group-lasso}
     Consider the class of validation loss functions $\mathcal{L}_{\text{val}} = \{\ell_{\alpha}^{\text{val}} : \mathcal{X} \to [-H,H] \mid \alpha \in \mathcal{A}\}$ induced by the Weighted Group Lasso objectives. The pseudo-dimension is bounded by $\operatorname{Pdim}(\mathcal{L}_{\text{val}}) = \mathcal{O}(p(d+p)^2 \log p)$.
\end{theorem}
The proof directly leverages our nested block elimination framework in Theorem~\ref{thm:general-pdim-upper-bound-via-block-eliminiation} by introducing auxiliary variables to encode the semi-algebraic $L_2$ norm as purely polynomial constraints. Prior work utilizing standard quantifier elimination yielded a looser upper bound of $\mathcal{O}(p^3d + p^2d^2)$ in \citet[Theorem 8.1]{le2026provably}. By preserving invariant connected sign cells during block elimination, Theorem~\ref{thm:data-driven-weighted-group-lasso} removes a full factor of $p$ from the dominant algebraic term, providing a strictly sharper generalization guarantee. See Appendix~\ref{apx:omitted-proof-applications} for further discussions on the improvement as well as the detailed proof. We also provided other applications to showcase the applicability of our general frameworks, which can be found in Appendix~\ref{apx:other-applications}.

\section{Conclusion and Future Work}
    In this paper, we studied the statistical complexity of multi-dimensional data-driven hyperparameter tuning for bi-level optimization. By replacing standard quantifier elimination with a nested block elimination framework based on invariant connected sign cells, we bypassed topological over-counting to derive strictly tighter pseudo-dimension bounds. For the training-loss setting, our novel multi-regime lower bounds proved that these upper bounds are tightly saturated with respect to both combinatorial and algebraic capacities. We further extended this framework to general validation-loss tuning and to broader semi-algebraic structures, such as the Weighted Group Lasso, thereby eliminating the inflated algebraic dependencies of prior work. While our lower bounds confirm the tightness of our guarantees in the training-loss setting ($f \equiv g$), determining the fundamental minimax lower bounds for the general bi-level validation-loss setting ($f \not\equiv g$) remains a compelling open problem for future research.

\bibliography{references}

\appendix 
\onecolumn
\section{Additional Definitions and Results} \label{apx:additional-defs-and-results}

\subsection{The Goldberg-Jerrum framework} 

The Goldberg-Jerrum (GJ) framework was introduced by \citet{goldberg1993bounding} and subsequently refined by \citet{bartlett2022generalization}. It provides a systematic mechanism to bound the pseudo-dimension of parameterized function classes. Specifically, it applies to any function class $\mathcal{L}$ where the evaluation of a function $\ell_\alpha$ can be modeled as a \textit{GJ algorithm}. Such algorithms are restricted to elementary arithmetic operations ($+, -, \times, \div$) and conditional branching, producing intermediate values that naturally take the form of rational functions of the input parameter $\alpha$. We formally define this computational model below:
    
    \begin{definition}[GJ algorithm, \cite{bartlett2022generalization}]\label{def:GJ-algorithm}
        A GJ algorithm $\Gamma$ operates on real-valued inputs and is restricted to two types of operations:\begin{itemize}\item Arithmetic assignments of the form $v'' = v \odot v'$, where $\odot \in \{ +, -, \times, \div\}$, and\item Conditional branching of the form ``if  $v \geq 0 \ldots$ else $\ldots$".\end{itemize}In both instances, the operands $v$ and $v'$ must either be external inputs or intermediate values previously generated by the algorithm.
    \end{definition}
    
    Because the intermediate variables $v, v', v''$ are generated solely through sequential arithmetic operations on the input $\alpha$, they inherently manifest as rational functions. The complexity of a given GJ algorithm is entirely characterized by the algebraic properties of these intermediate rational functions, formally captured by two metrics: \textit{degree} and \textit{predicate complexity}.
    
    \begin{definition}[Complexities of GJ algorithm, \cite{bartlett2022generalization}]\label{definition:gj-algorithm-complexities}
        The \textit{degree} of a GJ algorithm is the maximum degree of any rational function it computes based on its inputs. The \textit{predicate complexity} is defined as the total number of distinct rational functions evaluated within its conditional statements. For a rational function $f(\alpha) = \frac{g(\alpha)}{h(\alpha)}$, where $g$ and $h$ are polynomials in $\alpha$, its degree is given by $\deg(f) = \max\{\deg(g), \deg(h)\}$.
    \end{definition}
    
    When the evaluation of every function in a parameterized class $\mathcal{L}$ can be executed by a GJ algorithm with bounded algebraic and predicate complexities, we can invoke the following foundational theorem to explicitly bound its pseudo-dimension.
    
    \begin{theorem}[{\citet[Theorem~3.3]{bartlett2022generalization}}] \label{thm:gj}
        Suppose that each function $\ell_{\alpha} \in \mathcal{L}$ is specified by $p$ real parameters $\alpha \in \mathbb{R}^p$. Suppose further that for every problem instance $x \in \mathcal{X}$ and real-valued threshold $t \in \mathbb{R}$, there exists a GJ algorithm $\Gamma_{x, t}$ that takes $\alpha$ as input and returns true'' if $\ell_{\alpha}(x) \geq t$ and false'' otherwise. If $\Gamma_{x, t}$ exhibits a maximum degree of $\Delta$ and a predicate complexity of $\Lambda$, then $\text{Pdim}(\mathcal{L}) = \mathcal{O}(p\log(\Delta\Lambda))$.
    \end{theorem}
    
    In the next section, we will discuss in depth the idea in prior works, its limitation, and how we overcome this with the new block elimination argument. 
\subsection{A detailed discussion on the technical difference compared to prior works} \label{apx:detailed-technique-comparision}

\paragraph{Technique in prior works and its limitation.} In prior work, the standard approach to bounding the pseudo-dimension relies on a rigid two-step pipeline. First, the bi-level optimization objective is formulated as a polynomial First-Order Logic (FOL) statement, and a standard Quantifier Elimination (QE) algorithm is applied to convert it into a quantifier-free FOL. Second, the authors apply an off-the-shelf result—the Goldberg-Jerrum (GJ) framework—which treats this quantifier-free formula as a computational algorithm to derive the final pseudo-dimension bounds. However, the fundamental problem with this approach is that explicit QE is an exceptionally costly, worst-case algebraic procedure. To eliminate quantifiers, the algorithm forces the construction of a massive number of intermediate, redundant polynomial structures to algebraically separate every conceivable outcome. Because the off-the-shelf GJ framework blindly accounts for the quantities and degrees of all these newly generated, bloated polynomials, the resulting pseudo-dimension bounds suffer from severe algebraic inflation (such as the suboptimal $p^2$ dependency).

\paragraph{Our proposed technique, the difference, and how it resolves the limitation.}  In contrast, our proposed technique takes a more fundamental route. We recognize that the ultimate goal of bounding the pseudo-dimension does not actually require evaluating an explicit logic formula; it strictly requires bounding the total number of distinct sign patterns (the shattering coefficient) that the thresholded loss functions can realize across a dataset. Instead of algebraic manipulation, we use a geometric strategy based on nested block elimination to track how logical truth values remain invariant across topologically connected sign cells. By doing so, we completely bypass the need to ever construct the final quantifier-free FOL. Because we do not explicitly generate the redundant polynomial structures required by standard QE, we do not incur their associated complexity blowups. We simply bound the number of invariant sign regions directly, achieving the same ultimate goal—bounding the sign patterns—but doing so with drastically improved, tightly saturated pseudo-dimension bounds.

\subsection{Additional Definitions}
In this section, we will formally define the notion of \textit{cells} (or connected components).

\begin{definition}[Connectedness and connected sign cells \cite{basu2006algorithms}] \label{def:connected-component}
    Let $S \subseteq \bbR^d$ be a subset of $\bbR^d$. A subset $U \subseteq S$ is called \textit{open relative to $S$} if there exists an open set $O \subseteq \bbR^d$ such that $U = S \cap O$. The set $S$ is called \textit{connected} if there does \textit{not} exist two non-empty, disjoint subsets $U, V$, both open relative to $S$, such that $S = U \cup V$. A connected component $C$ of $S$ is a maximal connected subset of $S$, that is (i) $C$ is connected, and (ii) there is no connected set $C'$ such that $C \subsetneq C' \subseteq S$.
\end{definition}

\subsection{Supporting Lemmas}
The following inequality is a well-known and frequently used result in the learning theory literature. For completeness, we present its formal statement and detailed proof below. 
\begin{proposition} \label{prop:standard-ineq}
    For any $t \geq 0$ and $A \geq 2$, if $2^t \leq c(t + 2)A$ for a constant $c \geq 1$, then $t = \cO(\log A)$.
\end{proposition}
\proof 
    For convenience, we write $x = \frac{t}{2}$ for any $t \geq 0$. Note that $x \leq 2^x$ for all $x \geq 0$, we have
    \[
        t + 2 = 2(x + 1) \leq 2(2^x + 2^x) = 4\cdot 2^{t / 2}.
    \]
    From the assumption, we have $2^t \leq c(t + 2)A$. Combining with the above, we have
    \[
        2^t \leq 4c A2^{t/2} \Rightarrow 2^{t / 2} \leq 4Ac \Rightarrow t \leq 2\log_2 (4cA).
    \]
    Since $A \geq 2$, we have
    \[
        \log_2(4Ac) = \log_2 A+ \log_2(4c) = \cO(\log A).
    \]
    From the above, we have the final conclusion. 
\qed

\section{Additional Results and Omitted Proofs for Section \ref{sec:general-framework-main}} \label{apx:general-framework-main}

\begin{customprop}{\ref{prop:cellwise-logical-invariance} (restated)}
    Let $\Phi(z) = (q_1t^{(1)} \in \bbR^{d^{(1)}})\dots(q_Kt^{(K)} \in \bbR^{d^{(K)}}) \psi(z, t^{(1)}, \dots, t^{(K)})$ be a polynomial FOL, and let $\cP$ be the family of atomic polynomials appearing in the atomic predicates of $\psi$. If $\cQ$ is a nested sign-invariant projection of $\cP$ relative to the ordered blocks $t^{(1)}, \dots, t^{(K)}$, then $\Phi$ has a constant truth value on every connected sign cell $C \in \Cell(\cQ)$. In other words, for every $C \in \Cell(\cQ)$ and every $z_1, z_2 \in C$, we have $\Phi(z_1) \Leftrightarrow \Phi(z_2)$.
\end{customprop}

\proof
    For every sign condition $\eta = (\eta_1, \dots, \eta_s) \in \{-1, 0, 1\}^s$ Let $E_K(\eta)$ be the truth value obtained by evaluating every atomic predicate $P_i\chi_i 0$ using the polynomial sign $\eta_i$. Let $\Sigma_K, \Sigma_{K - 1}, \dots, \Sigma_{0}$ be the nested block-sign profiles of $\cP$ (as in Definition \ref{def:nested-block-sign-profiles}). Recursively, for $k = K, K - 1, \dots, 1$, and every possible value $S$ of the nested profile $\Sigma_{k - 1}$, we define
    \[
        E_{k - 1}(S) = \begin{cases}
            \bigvee_{\xi \in S}E_k(\xi), q_k = \exists,
            \\
            \bigwedge_{\xi \in S} E_k(\xi), q_k = \forall.
        \end{cases}
    \]
    Indeed, $\Sigma_{k - 1}$ records all the values of $\Sigma_k$ obtained by varying $t^{(k)}$. Therefore, an existential quantifier (e.g, $\exists$) takes the logical OR over the these values, whereas a universal quantifier (e.g., $\forall$) takes the logical AND. Applying this argument recursively from the innermost block $t^{(K)}$ to the outermost block $t^{(1)}$, the truth value of $\Phi(z)$ is $E_0(\Sigma_0(z))$. Since $\cQ$ is a nested sign invariant projection, by Definition~\ref{def:nested-sign-invariant-projection}, we have $\Sigma_0(z_1) = \Sigma_0(z_2)$, whenever $z_1, z_2$ belong to the same $C \in \Cell(\cQ)$. Therefore $E_0(\Sigma_0(z_1)) = E_0(\Sigma_0(z_2))$, meaning that the truth value of $\Phi(z_1)$ and $\Phi(z_2)$ is identical. 
\qed

\begin{customthm}{\ref{thm:general-pdim-upper-bound-via-block-eliminiation} (restated)}
    Consider a real-valued function class $\cF = \{f_\alpha: \cX \rightarrow \bbR \mid \alpha \in \cA\}$ be a real-valued function class, where $\cA = [\alpha_{\min}, \alpha_{\max}]^p \subseteq \bbR^p$. Suppose that there are integers $L, M, K \geq 1$, where $K$ is fixed, block dimensions $d_1, \dots, d_K \geq 1$, and a degree upper bound $\Delta \geq 1$ satisfying the following conditions: for any $x \in \cX$ and and real-valued threshold $\tau \in \bbR$, there exists
    \begin{enumerate}
        \item $L$ polynomial FOLs $\Phi_{x, \tau, 1}(\alpha), \dots, \Phi_{x, \tau, L}(\alpha)$, and
        \item a Boolean function $\beta_{x, \tau}: \{0, 1\}^L \rightarrow \{0, 1\}$
    \end{enumerate}
    such that $\bbI(f_\alpha(x) \geq \tau) = \beta_{x, \tau}(\Phi_{x, \tau, 1}(\alpha), \dots, \Phi_{x, \tau, L}(\alpha))$ for every $\alpha \in \cA$. Furthermore, assuming that every FOL $\Phi_{x, \tau, \ell}$:
    \begin{enumerate}
        \item has $K$ ordered quantified variables blocks of dimension $d_1, \dots, d_K$, respectively, and 
        \item contains at most $M$ distinct atomic polynomials, each of degree at most $\Delta$.
    \end{enumerate} 
     Define $A_K = \prod_{k = 1}^K(d_k + 1)$ and $B_K = \prod_{k = 1}^K d_k$, then the pseudo-dimension of the function class $\cF$ is at most 
     \[
        \Pdim(\cF) = \cO\left(p\log(2L) + pA_K \log (2M) + pB_K\log (2\Delta)\right).
     \]
\end{customthm}

\proof
    By the definition of pseudo-dimension $\Pdim(\cF)$, our goal is to bound the maximum size $N$ of a set of problem instances $S = \{x_1, \dots, x_N\}$ that $\cF$ can shattered and with real-valued thresholds $\tau = (\tau_1, \dots, \tau_N) \subset \bbR$ witness the shattering, that is, the binary vectors
    \[
        v((S, \tau), f_\alpha) = (\bbI(f_\alpha(x_1)\geq \tau_1), \dots, \bbI(f_\alpha(x_N)\geq \tau_N))
    \]
    admits all possible sign vectors $\{0, 1\}^N$ when varying $\alpha \in \cA$. This can be achieved by giving an upper bound for the \textit{shattering coefficient} $\cN(\cF, N)$ of the function class $\cF$ with respect to $N$, representing the maximum number of possible sign patterns $\{v(S, f_\alpha) \mid \alpha \in \cA\}$ for any $S \in \cX^N$ and $\tau \in \bbR^N$, and then solving for the maximum $N$ that satisfies the inequality $2^N \leq \cN(\cF, N)$. We will proceed in the following steps.

    \textbf{Step 1: Branch elimination.} For every $i = 1, \dots, N$ and $\ell = 1, \dots, L$, let $\cP_{i, \ell}$ be the set of distinct atomic polynomials appearing in $\Phi_{x_i, \tau_i, \ell}(\alpha)$. From assumption, we have: (1) $|\cP_{i, \ell}| \leq M$, and (2) $\max_{R \in \cP_{i, \ell}} \deg(R) \leq \Delta$. Let $c_K$ be the constant appearing in Theorem~\ref{thm:block-elimination}, and define $S_K = M^{A_K}\Delta^{c_KB_K}$ and $D_K = \Delta^{c_KB_K}$. By Theorem~\ref{thm:block-elimination}, for every pair $(i, \ell)$, there exists a nested sign-invariant projection $\cQ_{i, \ell} \subseteq \bbR[\alpha]$ (Definition~\ref{def:nested-sign-invariant-projection}) of $\cP_{i, \ell}$ such that: (1) $|\cQ_{i, \ell}| \leq S_K$, and (2) $\max_{R \in \cQ_{i, \ell}} \deg(R) \leq D_K$. By Proposition~\ref{prop:cellwise-logical-invariance}, the truth value of $\Phi_{x_i, \tau_i, \ell}(\alpha)$ is constant on every connected sign cell $C \in \Cell_{\cQ_{i, \ell}}$.

    \textbf{Step 2: Construct one global polynomial family}: Let $\cB_\cA = \{\alpha_j - \alpha_{\min}, \alpha_{\max} - \alpha_j \mid j = 1, \dots, p\}$ be the family of $2p$ affine polynomials defining the box-like region $\cA = [\alpha_{\min}, \alpha_{\max}]^p$. We then define the global polynomial set 
    \[
        \Tilde{\cQ} = \cB_\cA \cup \bigcup_{i = 1}^N\bigcup_{\ell = 1}^L \cQ_{i, \ell}.
    \]
    Let $q = \abs{\tilde{\cQ}}$. Because there are $NL$ branchwise projection families $\cQ_{i, \ell}$, each contains at most $S_K$ polynomials, we have $q \leq NLS_K + 2p$.Besides, every polynomial in $\tilde{\cQ}$ has degree at most $D_K$.

    \textbf{Step 3: The labeling vector is constant on every global cell}. Consider any connected sign cell $C \in \Cell(\tilde{\cQ})$. Fix an instance $x_i$ and a branch $\ell$. Since $\cQ_{i, \ell} \subseteq \tilde{\cQ}$, the sign of all polynomials in $\cQ_{i, \ell}$ are constant on $C$. Therefore, there exists a sign condition $\sigma_{i, \ell}$ on $\cQ_{i, \ell}$ such that $C \subseteq \Reali_{\cQ_{i, \ell}}(\sigma_{i, \ell)}$, and therefore in one cell belonging to $\Cell_{\cQ_{i, \ell}}$. Proposition~\ref{prop:cellwise-logical-invariance} now implies that the truth value of $\Phi_{x_i, \tau_i, \ell}(\alpha)$ is unchanged as $\alpha$ varies over $C$. 

    Note that the above holds for every $\ell = 1, \dots, L$. Therefore, the binary vector $(\Phi_{x_i, \tau_i, 1}(\alpha), \dots, \Phi_{x_N, \tau_N, 1}(\alpha))$ remains unchanged on $C$. Because $\beta_{x_i, \tau_i}$ is a fixed boolean function, the truth value of 
    \[
        \bbI(f_\alpha(x_i)\geq \tau_i) = \beta_{x_i, \tau_i}(\Phi_{x_i, \tau_i, 1}(\alpha), \dots, \Phi_{x_i, \tau_i, L}(\alpha))
    \]
    is also unchanged for $\alpha \in C$.

    Applying this argument to every $i = 1, \dots, N$, we claim that 
    \[
        v((S, \tau), f_\alpha) = (\bbI(f_\alpha(x_1)\geq \tau_1), \dots, \bbI(f_\alpha(x_N)\geq \tau_N))
    \]
    also remains unchanged for $\alpha \in C$. Therefore, the task now is to bound $|\Cell(\tilde{\cQ}|$ and solve for the maximum $N$ that satisfies $2^N \leq |\Cell(\tilde{\cQ}|$.

    \textbf{Step 4: Establish the pseudo-dimension upper bound.} The set of polynomials $\tilde{\cQ}$ contains $q \geq p$ polynomials in $p$ variables, each of degree at most $\Delta_K$. From Lemma~\ref{lm:cc-upper-bound}, there exists a universal constant $\kappa > 0$, such that $|\Cell(\tilde{\cQ}| \leq \left(\frac{\kappa q D_K}{p}\right)^p$. Note that $q \leq NLS_K + 2p$, our task now is to solve the inequality
    \begin{equation*}
        \begin{aligned}
            &2^N \leq  \left(\frac{\kappa (NLS_K + 2p) D_K}{p}\right)^p \\
            \Leftrightarrow &2^{N/p} \leq \left(\frac{\kappa (NLS_K + 2p) D_K}{p}\right).
        \end{aligned}
    \end{equation*}
    For convenience, let $u = \frac{N}{p}$, and the above becomes
    \begin{equation*}
        2^u \leq \kappa D_K(uLS_K + 2) \leq \kappa (u + 2)LS_KD_K.
    \end{equation*}
    Applying Proposition~\ref{prop:standard-ineq}, we have $u = \cO(\log(2LS_KD_K)$, or $N = \cO(p\log(2LS_KD_K)).$ Finally, substitute $S_K = M^{A_K}\Delta^{c_KB_K}$ and $D_K = \Delta^{c_KB_K}$ onto that, we have the final conclusion.
\qed

\section{Additional Results and Omitted Proofs for Section \ref{sec:training-loss}}

\subsection{Omitted Proofs for Section \ref{sec:training-loss-upper-bound}} \label{apx:training-loss-upper-bound}
In this section, we will present the detailed proof for Theorem~\ref{thm:training-loss-upper-bound}, which establishes the guarantee for data-driven hyperparameter tuning. We then provided a detailed discussion on the improvement of our results, compared to prior works as below. 
\begin{customthm}{\ref{thm:training-loss-upper-bound} (restated)}
    Let $\cA = [\alpha_{\min}, \alpha_{\max}]^p$ and $\Theta = [\theta_{\min}, \theta_{\max}]^d$ be the hyperparameter and parameter domains, respectively. Assume that for any problem instance $x \in \cX$, the training objective $f_x(\alpha, \theta)\triangleq f(x, \alpha, \theta)$, for any $(\alpha, \theta) \in \cA \times \Theta$, admits a piecewise polynomial structure with complexity $(M_f, T_f, \Delta_f)$. Then, the pseudo-dimension of the function class $\cL = \{\ell_\alpha: \cX \rightarrow \bbR \mid \alpha \in \cA\}$, where $\ell_\alpha(x) = \min_{\theta \in \Theta}f(x, \alpha, \theta)$, is at most
    \[
        \Pdim(\cL) = \Omega(p\log T_f + pd\log (M_f + d) + pd \log \Delta_f). 
    \]
\end{customthm}
\proof
    Consider any problem instance $x \in \cX$ and real-valued threshold $\tau \in \bbR$. Our goal is to show that: the binary value $\bbI(\ell_x(\alpha) \geq \tau)$, where $\ell_x(\alpha) \triangleq \ell_\alpha(x) = \min_{\theta \in \Theta}(x, \alpha, \theta)$, as a function of $\alpha$ can be rewritten as a boolean combination of at most $T_f$ polynomial FOLs $\Phi_{x, \tau, l}$, each satisfying the complexity conditions as in Theorem~\ref{thm:general-pdim-upper-bound-via-block-eliminiation}. This can be done as follows.

    \textbf{Step 1: Piecewise polynomial representation of $f_x(\alpha, \theta) \triangleq f(x, \alpha, \theta)$.} By assumption, for each problem instance $x$, there exists: (1) a set of boundary polynomials $\cH_x = \{h_{x, 1}, \dots, h_{x, M_x}\} \subset \bbR[\alpha, \theta]$ where $M_x \leq M_f$, (2) a set of realized sign patterns $\Sigma_x = \{-1, 0, 1\}^{m_x}$, where $|\Sigma_x| \leq T_f$, and (3) a set of piece polynomials $\{P_{x, \sigma}\}_{\sigma \in \Sigma_x} \subset \bbR[\alpha, \theta]$ such that every boundary polynomials $h_{x, m}$ (for $m = 1, \dots, M_x$) and piece polynomials $P_{x,\sigma}$ (for $\sigma \in \Sigma_x$) has degree at most $\Delta_f$. Those piece and boundary polynomials determine the form of $f_x(\alpha, \theta)$ as follows.

    For a sign pattern $\sigma = (\sigma_1, \dots, \sigma_{M_x}) \in \Sigma_x$, we define the region predicate 
    \[
        \textup{Where}_{x, \sigma}(\alpha, \theta) \triangleq \bigwedge_{m = 1}^{M_x}[\textup{sign}(h_{x, m}(\alpha, \theta) = \sigma_m].
    \]
    For each $(\alpha, \theta) \in \cA \times \Theta$, exact one sign pattern $\sigma \in \Sigma_x$ is active, and on its corresponding region, $f_x(\alpha, \theta)$ admits the polynomial forms $f_x(\alpha, \theta) = P_{x, \sigma}(\alpha, \theta)$. 

    \textbf{Step 2: Rewrite $\ell_x(\alpha)$ in the polynomial FOL form.} Recall that
    \[
        \ell_x(\alpha) = \min_{\theta \in \Theta} f_x(\alpha, \theta). 
    \]
    Therefore $\ell_x(\alpha) \geq \tau$ for a real-valued threshold $\tau$ if and only if for all $\theta \in \Theta$, we have $f_x(\alpha, \theta) \geq \tau$. Combining with Step 2, we can write the binary value $\bbI(\ell_x(\alpha) \geq \tau)$ as
    \begin{equation*}
        \begin{aligned}
            \bigwedge_{\sigma \in \Sigma_x} (\forall \theta \in \Theta)\left[\textup{Where}_{x, \sigma}(\alpha, \theta) \Rightarrow P_{x, \sigma}(\alpha, \theta) \geq \tau \right]
            = \bigwedge_{\sigma \in \Sigma_x} \Phi_{x, \tau, \sigma}(\alpha),
        \end{aligned}
    \end{equation*}
    where $\textup{Where}_{x, \sigma}(\alpha, \theta)$ is a logical statement defined as in Step 2, and
    \begin{equation*}
        \begin{aligned}
            &\Phi_{x, \tau, \sigma}(\alpha) = (\forall \theta \in \Theta) \left[\textup{Where}_{x, \sigma}(\alpha, \theta) \Rightarrow P_{x, \sigma}(\alpha, \sigma) \geq \tau \right] \\
            =&(\forall \theta \in \bbR^d)[\neg \textup{In}_{\Theta}(\theta) \lor \neg \textup{Where}_{x, \sigma}(\alpha, \theta) \lor P_{x, \sigma}(\alpha, \sigma) \geq \tau].
        \end{aligned}
    \end{equation*}
    Here, $\textup{In}_\Theta(\theta) = \bigwedge_{r = 1}^d[\theta_r \geq \theta_{\min}] \land [\Theta_{\max} \geq \theta_r]$ denotes the binary indicator if $\theta \in \Theta = [\theta_{\min}, \theta_{\max}]^d$, and in the above we use the logical identity that $(A \Rightarrow B) = (\neg A \lor B)$. We now see that $\bbI(\ell_{x}(\alpha) \geq \tau)$ is written in the form $\beta_{x, \tau} (\{\Phi_{x, \tau, 1}\}_{\sigma \in \Sigma})$, where $\beta_{x, \tau}(\{b_\sigma\}_{\sigma \in \Sigma_x}) = \bigwedge_{\sigma \in \Sigma_x}b_{\sigma}$, as required by Theorem~\ref{thm:general-pdim-upper-bound-via-block-eliminiation}.
    
    \textbf{Step 3: Count the complexity of each brand and apply Theorem~\ref{thm:general-pdim-upper-bound-via-block-eliminiation}.} The task left is to bound the complexity of $\Phi_{x, \tau, \sigma}$. Since each polynomial FOL contains exactly one quantified block $\forall \theta \in \bbR^d$, we have $K = 1$ and $d_1 = d$.

    Note that we have: (1) $M_x \leq M_f$ boundary polynomials $h_{x, 1}, \dots, h_{x, M_x}$, (2) $2d$ domain polynomials $\theta_r - \theta_{\min}$ and $\theta_{\max} - \theta_r$ for $r = 1, \dots, d$, and (3) a single value polynomial $P_{x, \sigma}(\alpha, \theta) - \tau$. Therefore, each brand contains at most $M_f + 2d + 1$ distinct atomic polynomial predicates. Finally, the maximum degree of all atomic polynomial predicates is simply at most $\Delta_f$. Finally, applying Theorem~\ref{thm:general-pdim-upper-bound-via-block-eliminiation}, we obtain
    \[
        \Pdim(\cL) = \cO(p \log T_f + pd\log (M_f + d) + pd \log \Delta_f
    \]
    as desired.
\qed

\begin{remark}
    Compared to prior work \citep{le2026provably}, whose standard quantifier-argument gives the upper bound
    \[
        \cO(pd \log(M_f + T_f + d) + p^2 d \log \Delta_f),
    \]
    our proposed result (Theorem~\ref{thm:training-loss-upper-bound}) instead gives
    \[
        \cO(p \log T_f + pd \log(M_f + d) + pd \log\Delta_f).
    \]
    Therefore, nested block elimination removes on factor $p$ from the degree-dependent term and decouples the number of pieces $T_f$ from the $d$-scaled logarithm. These improvements are supported by complementary lower bounds, which we will present in details below. 
\end{remark}

\subsection{Proofs for Section \ref{sec:training-loss-lower-bound}} \label{apx:training-loss-lower-bound}

In this section, we will formally present the proofs of Lemmas~\ref{lm:lower-bound-tf} and \ref{lm:lower-bound-mf}, which establish the lower-bound depending on the number of pieces $T_f$ and the number of boundary $M_f$. 


\begin{proof}[Proof of Lemma~\ref{lm:lower-bound-tf}]
    Our goal is to show that there exists $N = \Omega(p \log T_f)$ problem instances $x_1, \dots, x_N$ and real-valued threshold $\tau_1, \dots, \tau_N$ such that $\cL$ can shatter all the problem instances $x_i$ and $\tau_i$ (for $i = 1, \dots, N$) witnesses the shattering, that is
    \[
        |\{(\bbI(\ell_\alpha(x_1) \geq \tau_1), \dots, (\bbI(\ell_\alpha(x_N) \geq \tau_N) \mid \alpha \in \cA\}| = 2^N.
    \]

    Let $K = \lfloor \frac{T_f + 1}{2}\rfloor$, $B = \lfloor \log_2 K\rfloor$, $\cA = [0, K - 1]^p$, and $\Theta = [0, 1]^d \subset \bbR^d$ be a box-like region in $\bbR^d$. We construct $pB = \Omega(p \log T_f)$ problem instances $x_{j, b}$, for $j = 0, \dots, p - 1$ and $b = 0, \dots, B - 1$ by first defining the boundary polynomials for each problem instance: for each $x_{j, b}$, introduce $K - 1$ affine boundary polynomials:
    \[
        h_{j, q}(\alpha, \theta) = \alpha_j - \left(q + \frac{1}{2}\right), q = 0, \dots, K - 2.
    \]
    In other words, each boundary polynomial $h_{j, q}(\alpha, \theta)$ is just an affine function that does \textit{not} depend on the parameter $\theta$. The boundaries divide the hyperparameter $\alpha_j$ axis into $K$ open intervals $(-\infty, \frac{1}{2}), (\frac{1}{2}, \frac{3}{2}), \dots, (K - \frac{5}{2}, K - \frac{3}{2}), (K - \frac{3}{2}, +\infty)$ and $K - 1$ boundary points $\alpha_j = \frac{1}{2}, \frac{3}{2}, \dots, K - \frac{3}{2}$.

    The task now is to construct the piece function for each problem instance $x_{j, b}$. For convenience, we define the open interval containing $r \in \{0, \dots,  K - 1\}$ as $I_r$, and the sign condition realized in the interval $I_r$ as $\sigma_r$. Let $\textup{bit}_b(r)$ denote the $b^{th}$-binary bit of $r$. On the interval $I_r$, define
    \[
        P_{x_{j, b}, \sigma_r}(\alpha, \theta) = \frac{1}{2}\textup{bit}_b(r),
    \]
    which is simply a constant piece function. For any $\alpha_j$ that lies in one of $K - 1$ boundary points, we simply assign the value 1 to the corresponding piece functions. By this construction, the piece and boundary functions are independent of $\theta$, so minimizing over $\Theta$ changes nothing.

    Now, given an arbitrary sign pattern $y = (y_{j, b})_{j, b} \in \{0, 1\}^{pB}$, we simply choose $\alpha = (\alpha_0, \dots, \alpha_{p - 1})$ such that
    \[
        \alpha_j = \sum_{b = 0}^{B - 1} y_{j, b}2^{b}.
    \]
    From the construction, we claim that the function $f_{x_{j, b}}(\alpha, \theta)$ admits piecewise polynomial structure with complexity $(\lfloor \frac{T_f - 1}{2}\rfloor, T_f, 1)$ as expected. To see this, we have $K - 1 = \lfloor \frac{T_f - 1}{2}\rfloor$ boundary polynomials of the forms $ h_{j, q}(\alpha, \theta) = \alpha_j - \left(q + \frac{1}{2}\right)$ for $q = 0, \dots, K - 2$. Those boundaries induce $K$ open intervals and $K - 1$ boundary points, for a total of $2K - 1$ pieces. Therefore, the number of piece functions is at most
    \[
        |\Sigma_x| \leq 2 K - 1 = 2 \left\lfloor \frac{T_f + 1}{2} \right\rfloor - 1 \leq T_f.
    \]
    Finally, notice that all the piece and boundary functions are linear or constant functions with respect to $\alpha$ and $\theta$, so $\Delta_f = 1$.
    
    Moreover, we have the following claims:
    \begin{itemize}
        \item $0 \leq \alpha_j \leq 2^B - 1 \leq K - 1$ for any $j = 0, \dots, p - 1$. This means that $\alpha \in \cA = [0, K - 1]^p$.
        \item The open interval that contains $\alpha_j$ is $I_{\alpha_j}$, which assigns the sign condition $\sigma_{\alpha_j}$, and therefore
        \item the value of $\ell_{\alpha}(x_{j, b}) = P_{x_{j, b}, \sigma_{\alpha_j}}(\alpha, \theta) = \frac{1}{2}\textup{bit}_b(\alpha_j) = \frac{1}{2} y_{j, b}$.
    \end{itemize}
    Finally, choose the real-valued thresholds $\tau_{j, b} = \frac{1}{4}$ for all $j = 0, \dots, p - 1$ and $b = 0, \dots, B - 1$, we have
    \[
        \bbI(\ell_{\alpha}(x_{j, b}) \geq \tau_{j, b}) = y_{j, b}
    \]
    Therefore, we proved that for any sign pattern $y \in \{0, 1\}^{pB}$, there exists $\alpha$ such that 
    \[
        \bbI(\ell_\alpha(x_1) \geq \tau_1), \dots, (\bbI(\ell_\alpha(x_N) \geq \tau_N) = y,
    \]
    meaning that $\cL$ can shatter the set of problem instances $x_{j, b}$ (for $j = 0, \dots, p - 1$ and $b = 0, \dots, B - 1$) with $\tau_{j, b} = \frac{1}{4}$ witness the shattering. Since $pB = \Omega(p \log T_f)$, we claim that $\Pdim(\cL) = \Omega(p \log T_f)$
\end{proof}


\begin{proof}[Proof of Lemma~\ref{lm:lower-bound-mf}]
    Again, our goal is to construct $N = \Omega\left(pd \log \left(1 + \frac{M}{d}\right)\right)$ problem instances $x_1, \dots, x_N$ and real-valued thresholds $\tau_1, \dots, \tau_N$ such that 
    \[
        |\{(\bbI(\ell_\alpha(x_1) \geq \tau_1), \dots, (\bbI(\ell_\alpha(x_N) \geq \tau_N) \mid \alpha \in \cA\}| = 2^N.
    \]

    Let $s \triangleq \min \{M, d\}$, $a \triangleq \left\lfloor \frac{M}{s} \right\rfloor$, $\rho \triangleq M - as \in \{0, \dots, s - 1\}$. We then define
    \[
        m_i \triangleq \begin{cases}
            a + 1, & 1 \leq i \leq \rho, \\
            a, &\rho < i \leq s,\\
            0 & s < i \leq d.
        \end{cases}
    \]
    Then $\sum_{i = 1}^d m_i = M$. We then define
    \[
        R_{M, d} \triangleq \prod_{i = 1}^d (m_i + 1), \quad S_{m, d}\triangleq \prod_{i = 1}^d (2m_i + 1).
    \]
    Roughly speaking, $R_{M, d}$ is the number of sign regions in which none of the boundary polynomials vanishes, while $S_{M, d}$ is the total number of realized sign patterns, including patterns with boundary equalities. 

    \textbf{Step 1: Defining the \textit{shared} boundary functions.} For every coordinate $i = 1, \dots, d$ and every $q = 0, \dots, m_i - 1$, we define the boundary functions $h_{i, q}(\alpha, \theta)$ for all problem instances as 
    \[
        h_{i, q}(\alpha, \theta) \triangleq \theta_i - \left(q + \frac{1}{2}\right).
    \]
    By construction, there are exactly $\sum_{i = 1}^d m_i = M_f$ such boundary affine polynomials, and do \textit{not} depend on $\alpha$. We define the parameter box-like region as $\Theta \triangleq \prod_{i = 1}^d [-1, m_i + 1]$. For a fixed coordinate $i$, the ordered thresholds $\frac{1}{2}, \frac{3}{2}, \dots, m_i - \frac{1}{2}$ can produce at most $m_i + 1$ non-zero sign patterns, and $m_i$ zero-containing sign patterns, one for each threshold equality. Therefore, the one-dimensional family $\{h_{i, q}\}_{q = 0}^{m_i - 1}$ realizes exactly $2 m_i + 1$ sign patterns on the $i^{th}$ coordinate  $\theta_i$ of $\theta \in \Theta$.

    Note that the ordinates $\theta_i$ are independent: any choice of one realizable coordinate-wise sign pattern is obtained by choosing the coordinate $\theta_i$ separately. Therefore, the full boundary family $\{h_{i, q}(\alpha, \theta)\}_{i, q}$ realize exactly $S_{M, d} = \prod_{i = 1}^d (2m_i + 1)$ sign pattern when varying $\theta \in \Theta$.

    Moreover, among these patterns, there are exactly $m_i + 1$ coordinate-wise choices that contain no zero for coordinate $i$. Therefore, the full product regions containing no boundary equality are $R_{M, d} = \prod_{i = 1}^d (m_i + 1)$. Each such region corresponds to a distinct integer vector $u = (u_1, \dots, u_d)$, where $u_i \in \{0, \dots, m_i\}$. Thus, every one of these $R_{M,d}$ regions is non-empty and belongs to $\Theta$.

    \textbf{Step 2: Define the problem instances via the piecewise functions.} For convenience, define $B \triangleq \lfloor \log_2 R_{M, d}\rfloor$ and $K \triangleq 2^B$. Because $K \leq R_{M, d}$, we may select $K$ distinct zero-free grid regions and enumerate the as $C_0, \dots, C_{K - 1}$. Let $\sigma_r$ be the boundary sign pattern corresponding to $C_r$, and let $\cA = [0, K - 1]^p$ be the box-like hyperparameter regions. We then define a finite problem instance sets $\cX:= \{x_{j, b}: j \in \{1, \dots, p\}, b \in \{0, \dots, B - 1\}\}$, and therefore $\abs{cX} = pB$. For $r \in \{0, \dots, K - 1\}$, let $\textup{bit}_b(r) \in \{0, 1\}$ denote the $b^{th}$ binary bit of $r$.

    We formally construct the problem instance $x_{j, b}$ as follows. On the selected sign pattern $\sigma_r$, we define the piece function $P_{x_{j, b}, \sigma_r}(\alpha, \theta)$ as
    \[
        P_{x_{j, b}, \sigma_r}(\alpha, \theta) \triangleq (\alpha_j - r)^2 + \frac{1}{2}\textup{bit}_b(r), \quad r = 0, \dots, K - 1.
    \]
    For all other realized sign patterns $\sigma$, including all boundary patterns and all unselected zero-free patterns, we simply define their corresponding piece functions as
    \[
        P_{x_{j, b}, \sigma}(\alpha, \theta) \triangleq 2.
    \]

    By construction, each problem instances $x_{j, b}$ have the function  $f_{j, b}$ that admits \textit{piecewise polynomial structure} with complexity $(M, S_{M, d}, 2)$.

    \textbf{Step 3: Establish the pseudo-dimension lower-bound}. We now show that the constructed problem instances $x_{j, b}$ can be shattered by $\cL$, and there are appropriate real-valued thresholds $\tau_{j, b}$ that witness the shattering.

    Consider an arbitrary labeling $y = (y_{j, b})_{j, b} \in \{0, 1\}^{pB}$. For any $j = 1, \dots, p$, we define
    \[
        r_j \triangleq \sum_{b = 0}^{B - 1} y_{j, b}2^b \in \{0, \dots, 2^B - 1\} = \{0, \dots, K - 1\},
    \]
    and choose the hyperparameter $\alpha = (\alpha_1, \dots, \alpha_p)$ such as $\alpha_j \triangleq r_j$. Then $\alpha \in \cA$.

    Consider a problem instance $x_{j, b}$. Since $C_{r_j}$ is not empty, there exists $\theta \in C_{r_j} \cap \Theta$, and on this region, we have
    \[
        P_{x_{j, b}, \sigma_{r_j}}(\alpha, \theta) = (r_j - r_j)^2 + \frac{1}{2}\textup{bit}_{b}(r_j) = \frac{1}{2}y_{j, b}.
    \]
    For any other selected region where $s \neq r_j$, we have
    \[
        P_{x_{j, b}, \sigma_s}(\alpha, \theta) = (r_j - s)^2 + \frac{1}{2}\textup{bit}_{b}(s) \geq 1.
    \]
    The inequality comes from the fact that $r_j, s$ are integer, meaning $(r_j - s)^2 \geq 1$ if $r_j \neq s$. This implies that 
    \[
        \ell_\alpha(x_{j,b}) = \min_{\theta \in \Theta}f(x_{j, b}, \alpha, \theta) = \frac{1}{2}y_{j, b}.
    \]
    Finally, choose $\tau_{j, b} = \frac{1}{4}$ for all $j \in \{1, \dots, p\}$ and $b \in \{0, \dots, B - 1\}$, we have 
    \[
        \bbI\left(\ell_\alpha(x_{j, b}) \geq \tau_{j, b}\right) = y_{j, b}.
    \]
    Therefore, $\cL$ can shatter the set $\{x_{j, b}\}_{j,b}$ of $pB$ problem instances with $\{\tau_{j, b}\}_{j,b}$ witnesses the shattering. Therefore, we claim that $\Pdim(\cL) \geq pB = p \left\lfloor \log_2 R_{M, d}\right\rfloor$.

    the task now is to show that $\log R_{M, d} = \Omega\left(d \log \left(1 + \frac{M}{d}\right)\right)$. We consider two cases:
    \begin{itemize}
        \item If $M \leq d$, then $R_{M, d} = 2^M$. Since $\log(1 + t) \leq t$ for $t \geq 0$, we have $d\log (1 + \frac{M}{d}) \leq M$, whereas $\log M_{M,d} = M \log 2$. This means $\log R_{M, d} = \Omega\left(d \log \left(1 + \frac{M}{d}\right)\right)$ holds true.

        \item If $M > d$, then $s = d$, and $a = \lfloor M/d\rfloor \geq 1$. By definition, we have $R_{M, d} \geq (a + 1)^d$. Let $x = M/d \geq 1$, we claim that $a + 1 \geq \sqrt{1 + x}$. To see that, if $1 \leq x < 2$, we have $a + 1 = 2 \geq \sqrt{1 + x}$. If $x \geq 2$, then $a + 1 > x  \geq \sqrt{1 + x}$. Therefore
        \[
            \log R_{M,d} \geq d \log (a + 1) \geq \frac{d}{2}\log (1 + x) = \frac{d}{2}\log \left(1 + \frac{M}{d}\right).
        \]
        This also implies $\log R_{M, d} = \Omega\left(d \log \left(1 + \frac{M}{d}\right)\right)$ holds true.
    \end{itemize}

    Therefore, we claim that $\Pdim(\cL) = \Omega\left(pd \log\left(1 + \frac{M_f}{d}\right)\right)$.
\end{proof}

\section{Additional Results and Proofs for Section~\ref{sec:validation-loss}} \label{apx:validation-loss}

In this section, we will formally present the proof for Theorem~\ref{thm:validation-loss-upper-bound}. Again, we then provide a discussion on the improvement of our proposed result compared to prior work by Le et al. \cite{le2026provably}.

\begin{customthm}{\ref{thm:validation-loss-upper-bound} (restated)}
    Under Assumption 1, let $\mathcal{L}_{\text{val}} = \{\ell_{\alpha}^{\text{val}} : \mathcal{X} \to [-H,H] \mid \alpha \in \mathcal{A}\}$, $\Delta_{f,g} = \max\{1, \Delta_f, \Delta_g\}$, and $M_{\textup{tot}} = M_f + M_g + T_f + d.$
    \[
        \operatorname{Pdim}(\mathcal{L}_{\text{val}}) = \mathcal{O}\left(p \log(T_f T_g) + p d^2 \log\left(2 + M_{\textup{tot}}\Delta_{f,g}\right)\right).
    \]
\end{customthm}
\proof
    Fixed a problem instance $x \in \cX$ and a real-valued threshold $\tau\in \bbR$, our goal is to express $\bbI(\ell_{\alpha}^\textup{val}(x) \geq \tau)$ as a boolean combination of of polynomial FOL to which Theorem~\ref{thm:general-pdim-upper-bound-via-block-eliminiation} applies. Recall that
    \[
        \cS(x, \alpha) = \arg \min_{\theta \in \Theta} f_x(\alpha, \theta), \textup{ and }\ell_\alpha^\textup{val} = \inf_{\theta \in \cS(x, \alpha)} g_x(\alpha, \theta).
    \]
    We will proceed with the following steps.
 
    \noindent \textbf{Step 1: Piecewise polynomial representation.} For the given fixed problem instances, let $\Sigma^f_{x} \subset \{-1, 0, 1\}^{M_f}$, and $\Sigma^g_{x} \subset \{-1, 0, 1\}^{M_g}$ be the sets of of sign conditions indexing the training and validation pieces, respectively. By Assumption~\ref{asmp:structural}, we have $|\Sigma_x^f| \leq T_f$ and $|\Sigma^g_x| \leq T_g$. For every $\sigma \in \Sigma_x^f$, let $F_{x, \sigma}$ be the corresponding training piece polynomial. For every $\gamma \in \Sigma^g_x$, let $G_{x, \gamma}$ be the corresponding validation piece polynomial. We then define the region predicates as
    \[
        \textup{Where}^f_{x, \sigma}(\alpha, \theta) = \bigwedge_{m = 1}^{M_f}[\sign (h'_{x, m}(\alpha, \theta) = \sigma_m],
    \]
    and
    \[
        \textup{Where}_{x, \gamma}^g(\alpha, \theta) = \bigwedge_{m = 1}^{M_g}[\sign(h^g_{x, m}(\alpha, \theta) = \gamma_m].
    \]
    Thus, whenever $\textup{Where}^f_{x, \sigma}(\alpha, \theta)$ holds, we have
    \[
        f_x(\alpha, \theta) = F_{x, \sigma}(\alpha, \theta), 
    \]
    and similarly, whenever $\textup{Where}^f_{x, \sigma}(\alpha, \theta)$ holds, we have
    \[  
        g_x(\alpha, \theta) = G_{x, \gamma}(\alpha, \theta).
    \]
    Moreover, we also define the parameter-feasibility checking condition as
    \[
        \textup{In}_\Theta(u) = \bigwedge_{r = 1}^d[u_r - \theta_{\min} \geq 0 \land \theta_{\max} - u_r \geq 0],
    \]
    which is true exactly when $u \in \Theta$.

    \noindent \textbf{Step 2: Certifying that a candidate is a training minimizer.} 
    Fixed $\sigma \in \Sigma^f_x$. For a candidate solution $\theta$ and a competing solution $u$, we define
    \[
        \textup{Compare}_{x, \sigma}(\alpha, \theta, u) \triangleq \neg \textup{In}_\Theta(u) \lor \bigwedge_{\rho \in \Sigma^f_x}\left[ \neg \textup{Where}^f_{x, \rho}(\alpha, u) \lor  F_{x, \sigma}(\alpha, \theta) \leq F_{x, \rho}(\alpha, u) \right].
    \]
    Let's elaborate the predicate above closely. First, suppose that $u \in \Theta$. Note that there is exactly one sign conditions $\rho \in \Sigma_x^f$ is active at  the point $(\alpha, u) \in \cA \times \Theta$. For that active $\rho$, the predicate requires
    \[
        F_{x, \sigma}(\alpha, \theta) \leq F_{x, \rho}(\alpha, u) = f_x(\alpha, u).
    \]
    For all other inactive $\rho$, the corresponding implication is automatically true. Therefore, provided in that $\theta$ belong to the training corresponding to $\sigma$, we have $(\forall u \in \bbR^d) \textup{Compare}_{x, \sigma}(\alpha, \theta, u)$ is equivalent to 
    \[
        f_x(\alpha, \theta) \leq f_x(\alpha, u) \textup{ for all } u \in \Theta.
    \]
    Therefore
    \[  
        \theta \in \cS(x, \alpha) \Leftrightarrow \textup{In}_\Theta(\theta) \land \textup{Where}_{x, \sigma}^f (\alpha, \theta) \land (\forall u\in \bbR^d) \textup{Compare}_{x, \sigma}(\alpha, \theta, u),
    \]
    for the active training sign condition $\sigma$.
 
    \noindent \textbf{Step 3: Encoding a bad training minimizer.} For each pair $(\sigma, \gamma) \in \Sigma_x^f \times \Sigma^g_x$, we define $\Phi_{x, \tau, \sigma, \gamma}(\alpha)$ as 
    \begin{equation*}
        \begin{aligned}
           (\exists \theta \in \bbR^d) \left[\textup{In}_\Theta(\theta) \land \textup{Where}^f_{x, \sigma}(\alpha, \theta) \land \textup{Where}_{x, \gamma}^g(\alpha, \theta) \land G_{x, \gamma}(\alpha, \theta) < \tau \land (\forall u \in \bbR^d) \textup{Compare}_{x, \sigma}(\alpha, \theta, u)\right].
        \end{aligned}
    \end{equation*}
    Because the candidate conditions do not depend on $u$, the formula above can be equivalently be written in the form with the two ordered blocks $(\exists \theta \in \bbR^d) (\forall u \in \bbR^d)$. This implies that every branch has $K = 2$ and $d_1 = d_2 = d$.

    \noindent \textbf{Step 4. Recovering the validation-loss threshold event.} Since $\cS(x, \alpha) \neq \emptyset$, we have $\inf_{\theta \in \cS(x, \alpha)} g_x(\alpha, \theta) < \tau$ if and only if there exists $\theta \in \cS(x, \alpha)$ such that $g_x(\alpha, \theta) < \tau$. Importantly, this does not require the infimum to be attained. If the infimum is strictly below $\tau$, the definition of the infimum guarantees the existence of an element with value below $\tau$. It follows that
    \[
        \ell_\alpha^\textup{val}(x) < \tau \Leftrightarrow \bigvee_{\sigma \in \Sigma^f_g, \gamma \in \Sigma_x^g} \Phi_{x, \tau, \sigma, \gamma}(\alpha).
    \]
    This implies that the condition $\bbI(\ell_\alpha^\textup{eval}(x) \geq \tau)$ is a NOR-type Boolean combination of at most $L \leq T_fT_g$ polynomial branches, as required by Theorem~\ref{thm:general-pdim-upper-bound-via-block-eliminiation}.

    \noindent \textbf{Step 5. Counting the atomic polynomials and applying Theorem~\ref{thm:general-pdim-upper-bound-via-block-eliminiation}}. Finally, note that each branch contains at most $M_\textup{tot} = 2M_f + M_g + T_f + 4d + 1$ distinct atomic polynomial predicate, and each atomic predicate has the degree at most $\Delta_{f, g} = \max \{1, \Delta_f, \Delta_g\}$. Substituting into Theorem~\ref{thm:general-pdim-upper-bound-via-block-eliminiation} we have the final conclusion. 
\qed

\begin{remark}[Comparison to prior work]
    Compared to \cite{le2026provably}, Theorem 6.1, which obtained the upper bound of
    \[
        \cO(pd^2 \log (M_f + T_F + M_g + T_g + d) + p^2d^2 \log \Delta_{f, g}),
    \]
    using standard quantifier elimination technique, our Theorem~\ref{thm:validation-loss-upper-bound} improves both algebraic and combinatorial dependencies. In particular, our technique using \textit{nested block elimination} (Theorem~\ref{thm:general-pdim-upper-bound-via-block-eliminiation}) reduces the degree-dependent term from $p^2 d^2 \log \Delta_{f, g}$ to $pd^2 \log \Delta_{f, g}$. Moreover, by treating the $T_fT_g$ possible pairs of active training and validation pieces as separate logical branches, the dependence on $T_g$ decreases from $pd^2 \log T_g$ to just $p\log T_g$. The training-piece parameter $T_f$ remains in $M_{tot}$ because certifying lower-level optimality requires comparison against every possible training piece. 
\end{remark}

\section{Additional results and Omitted proofs for Section~\ref{sec:applications}} \label{apx:applications}

\subsection{Omitted proof for Section~\ref{sec:applications}} \label{apx:omitted-proof-applications}
    In this section, we present the detailed proof for Theorem~\ref{thm:data-driven-weighted-group-lasso}. 

    \begin{customthm}{\ref{thm:data-driven-weighted-group-lasso} (restated)}
        Consider the class of validation loss functions $\mathcal{L}_{\text{val}} = \{\ell_{\alpha}^{\text{val}} : \mathcal{X} \to [-H,H] \mid \alpha \in \mathcal{A}\}$ induced by the Weighted Group Lasso objectives. Assuming that $\Theta = \bbR^d$ and $\cA \subset (0, \infty)^p$, then the pseudo-dimension is bounded by $\operatorname{Pdim}(\mathcal{L}_{\text{val}}) = \mathcal{O}(p(d+p)^2 \log p)$.
    \end{customthm}
    \proof
        Assume under standard unconstrained formulation, that is $\Theta = \bbR^d$ and $\cA \subset (0, \infty)^p$. For $v = (v_1, \dots, v_p) \in \bbR^d$ and $\nu \in \bbR^p$, we define
        \[
            \textup{Norm}(v, \nu) = \sum_{i = 1}^p \left[ \nu_i \geq 0 \land \nu_i^2 = \sum_{j = 1}^{d_i} v_{i, j}^2 \right].
        \]
        In other words, the term $\textup{Norm}(u, \nu)$ holds exactly when $\nu_i = \|v_i\|^2$ for every $i$. We then define the polynomial lifting of the training objective as
        \[
            \tilde{f}_x(\alpha, v, \nu) = \|A\nu - b\|_2^2 + \sum_{i = 1}^p \alpha_i \nu_i.
        \]
        Given a problem instance $x$ and a real-valued threshold $\tau \in \bbR$, consider the polynomial FOL formula
        \[
            \Phi_{x, \tau}(\alpha) \triangleq (\forall \theta \in \bbR^d) (\exists (z, \nu^\theta, \nu^z) \in \bbR^{d + 2p})\left[g_x(\theta) \geq \tau \lor \left(\textup{Norm}(\theta, \nu^\theta) \land \textup{Norm}(z, \nu^z) \land \tilde{f}_x(\alpha, z, \nu^z) < \Tilde{f}_x(\alpha, \theta, \nu^\theta)\right)\right].
        \]
        In words, this formula states that every $\theta$ either has a validation loss at least $\tau$, or admits another point $z$ which strictly smaller training loss. Therefore
        \[
            \Phi_{x, \tau}(\alpha) \Leftrightarrow \textup{every }\theta \in \cS(x, \alpha) \textup{ satisfies }g_x(\alpha) \geq \tau.
        \]
        Indeed, a training minimizer cannot satisfy the second disjunct, while every non-minimizer admits a strictly better point because $\cS(x, \alpha) \neq \emptyset$. Consequently,
        \[
            \Phi_{x, \tau}(\alpha) \Leftrightarrow \inf_{\theta \in \cS(x, \alpha)} g_x(\theta)\geq \tau,
        \]
        and therefore
        \[
            \bbI(\ell_\alpha^{\textup{val}}(x) \geq \tau) = \Phi_{x, \tau}(\alpha).
        \]
        This one has one branch and two quantified blocks, with $L = 1$, $K = 2$, $d_1 = d$, and $d_2 = d + 2p$. Substituting to Theorem~\ref{thm:general-pdim-upper-bound-via-block-eliminiation}, we have the final conclusion. 
    \qed

    \begin{remark}[Comparison with prior work]        
         Le et al. \cite{le2026provably} (Theorem 8.1) obtained the pseudo-dimension bound ($\cO(p^{3}d+p^{2}d^{2})$) for data-driven Weighted Group Lasso using standard quantifier elimination. Because the (p) groups are nonempty and (e.g., $\cO\left(\sum_{i=1}^{p}d_i=d\right)$), we necessarily have ($p\le d$), and their bound simplifies to (e.g., $\cO (p^{2}d^{2})$). In contrast, Theorem~\ref{thm:data-driven-weighted-group-lasso} gives $\cO\left(p(d+p)^{2}\log(2p)\right)= \cO\left(pd^{2}\log(2p)\right))$. Thus, our result replaces the quadratic dependence on the number of groups in the dominant term by an essentially linear dependence, improving the previous bound by a factor of order (that is, $\frac{p}{\log(2p)}$) as $p$ grows. Both analyses convert the group norms into polynomial form using auxiliary variables; the improvement therefore does not arise from a different representation of the Weighted Group Lasso objective. Rather, it follows from the nested block-elimination argument in Theorem 4.2, which avoids the inflated algebraic dependence introduced by standard quantifier elimination. For fixed $p$, both bounds retain the same ($d^{2}$) dependence, whereas the improvement becomes increasingly significant when the number of groups grows.
    \end{remark}
\subsection{Other Applications} \label{apx:other-applications}
    We note that our main contribution is the an novel \textit{general} framework to establish statistical learning guarantees for a class of data-driven algorithms design problems which admits piecewise polynomials structure, or even more general, the structure that can be converted into the form of Theorem~\ref{thm:general-pdim-upper-bound-via-block-eliminiation}. To further showcase the applicability of our general frameworks, in this section we will provide other applications. 

\subsubsection{Data-driven Cost-sensitive SVM.}

Classification errors frequently have heterogeneous costs across classes or subpopulations, motivating cost-sensitive extensions of support vector machines. Rather than fixing these costs manually, we treat the group-specific penalty vector as a data-driven algorithm configuration learned across problem instances. The hinge objective is piecewise polynomial: its polynomial representation changes whenever a training or validation example crosses its affine margin boundary. Consequently, cost-sensitive SVM provides a direct application of our connected-sign-cell framework, complementary to Weighted Group Lasso, which instead illustrates algebraic lifting.

Fix integers $n, m, p, d \geq 1$. A problem instance $x = (D^\textup{tr}_x, D^\textup{val}_x)$ contains a training set $D_x^{\text{tr}} = \{(a_{x,j}, y_{x,j}, c_{x,j})\}_{j=1}^n$ and a validation set $D_x^{\text{val}} = \{(a'_{x,k}, y'_{x,k})\}_{k=1}^m$. Here $a_{x,j}, a'_{x,k} \in \mathbb{R}^d, \quad y_{x,j}, y'_{x,k} \in \{-1, 1\},$ and $c_{x,j} \in \{1, \dots, p\}$ identifies the cost group of training example $j$. Let $\mathcal{A} = [\underline{C}, \overline{C}]^p \subset (0, \infty)^p$ be the domain of group-specific penalty parameters, and let $\Theta = [-B, B]^d$ be the model-parameter domain. For a fixed regularization coefficient $\lambda > 0$, define
\[
    f_x(\alpha, w) = \frac{\lambda}{2}\Vert{}w\Vert{}_2^2 + \frac{1}{n} \sum_{j=1}^n \alpha_{c_{x,j}} \left[ 1 - y_{x,j} a_{x,j}^\top w \right]_+,
\]
where $[r]_+ = \max\{0, r\}.$ We use the ordinary validation hinge loss
\[
    g_x(w) = \frac{1}{m} \sum_{k=1}^m \left[ 1 - y'_{x,k} {a'_{x,k}}^\top w \right]_+.
\]
The induced optimistic validation loss is $\ell_\alpha^{\text{svm}}(x) = \inf_{w \in S(x, \alpha)} g_x(w),$ where $\cS(x, \alpha) = \arg\min_{w \in \Theta} f_x(\alpha, w).$ Because $\Theta$ is compact and $f_x$ is continuous, $S(x, \alpha)$ is nonempty. Moreover, the quadratic regularizer makes $f_x(\alpha, \cdot)$ strongly convex, so the minimizer is unique. The general framework, however, does not require uniqueness.The box $\Theta$ can be chosen sufficiently large that it does not modify the ordinary unconstrained SVM solution. Indeed, if $w^\star$ is the unconstrained minimizer, then
\[
    \frac{\lambda}{2} \Vert{}w^\star\Vert{}_2^2 \le f_x(\alpha, w^\star) \le f_x(\alpha, 0) \le \overline{C}.
\]
Therefore $\Vert{}w^\star\Vert{}_2 \le \sqrt{\frac{2\overline{C}}{\lambda}},$ so it suffices to take $B \ge \sqrt{\frac{2\overline{C}}{\lambda}}.$

For each training example, define the affine margin polynomial $h_{x,j}^f(w) = 1 - y_{x,j} a_{x,j}^\top w, \quad j = 1, \dots, n.$ Let $T_{\text{tr}}$ be a uniform upper bound, over $x$, on the number of sign conditions of $\{ h_{x,1}^f, \dots, h_{x,n}^f \}$ realized on $\Theta$. Similarly, define $h_{x,k}^g(w) = 1 - y'_{x,k} {a'_{x,k}}^\top w, \quad k = 1, \dots, m,$ and let $T_{\text{val}}$ uniformly bound the number of realized validation sign conditions. We then have the following result, which establishes generalization guarantee for the problem of data-driven cost-sensitive SVM. 

\begin{corollary}[Data-driven cost-sensitive SVM] \label{corol:pseudo-svm}
    Let $\mathcal{L}_{\text{svm}} = \{ \ell_\alpha^{\text{svm}} : \mathcal{X} \to \mathbb{R} \mid \alpha \in \mathcal{A} \}.$ We then have $\text{Pdim}(\mathcal{L}_{\text{svm}}) = O\left( p(n + m) + p d^2 \left( n + \log(d + m + 2) \right) \right).$
\end{corollary}
\begin{proof}[Proof of Corollary~\ref{corol:pseudo-svm}]
    Fix an instance $x$. The training objective has the $n$ affine boundary polynomials $h_{x,1}^f, \dots, h_{x,n}^f.$ For a realized sign condition $\sigma \in \{-1, 0, 1\}^n,$ we define $I_+(\sigma) = \{ j \in \{1, \dots, n\} : \sigma_j = 1 \}.$ On the realization of $\sigma$, the hinge terms indexed by $j \notin I_+(\sigma)$ vanish, while the remaining hinge terms equal their affine arguments. Therefore, the active training piece is 
    \[
        F_{x,\sigma}(\alpha, w) = \frac{\lambda}{2}\Vert{}w\Vert{}_2^2 + \frac{1}{n} \sum_{j \in I_+(\sigma)} \alpha_{c_{x,j}} \left( 1 - y_{x,j} a_{x,j}^\top w \right).
    \]
    This is a polynomial of total degree at most two in $(\alpha, w)$. In particular, the only mixed terms have the form $\alpha_{c_{x,j}} w_r,$ which are bilinear. It follows that $f_x$ has piecewise-polynomial complexity $(M_f, T_f, \Delta_f) = (n, T_{\text{tr}}, 2)$, where $T_{\text{tr}} \leq 3^n$. Similarly, fo the validation objective, we conclude that $g_x$ has the piecewise polynomial complexity $(M_g, T_g, \Delta_g) = (m, T_{\text{val}}, 1)$, where $T_\textup{val} \leq 3^m$. Besides, we note that $\Delta_{f,g} = \max\{1, \Delta_f, \Delta_g\} = 2,$ and $M_{\text{tot}} = M_f + M_g + T_f + d = n + m + T_{\text{tr}} + d.$ Substituting to Theorem~\ref{thm:validation-loss-upper-bound}, we have the final conclusion. 
\end{proof}
 
\subsubsection{Multi-penalty ridge regression}

Multi-parameter ridge regularization uses several penalties to promote different structural properties of the fitted solution. Let 
\[
    f_x(\alpha, \theta) = \|A\theta - b\|_2^2 + \sum_{i = 1}^p \alpha_i \|D_i \theta\|_2^2,
\]
and let $g_x(\alpha, \theta) = \|A' \theta - b'\|_2^2$, where $\alpha \in \cA = [\alpha_{\min}, \alpha_{\max}]^p$ and $\theta \in \Theta = [\theta_{\min}, \theta_{\max}]^d$ are box like hyperparameters and parameters domain. Let $\cL_\textup{ridge}$ denote the corresponding optimistic validation-loss class. We then have the following result, which establishes the pseudo-dimension upper bound for the problem of tuning regularization parameters in multi-penalty ridge regression.

\begin{corollary}[Pseudo-dimension for ridge regression] \label{corollary:psedo-ridge}
    We have $\Pdim(\cL_\textup{ridge}) = \cO(pd^2 \log d)).$
\end{corollary}
\begin{proof}[Proof of Corollary~\ref{corollary:psedo-ridge}]
    It is clearly that the complexity of the training and validation objectives are $(M_f, T_f, \delta_f) = (0, 1, 3)$, and $(M_g, T_g, \Delta_g) = (0, 1, 2)$, respectively. Therefore applying Theorem~\ref{thm:validation-loss-upper-bound} gives us the final conclusion. 
\end{proof}


\end{document}